\documentclass{article}

\usepackage[top=3cm, bottom=3cm, left=2.5cm, right=2.5cm]{geometry}
\usepackage{natbib}

\usepackage{comment}

\usepackage{tikz}
\usetikzlibrary{matrix}
\usepackage{hyperref} 
    \hypersetup{
    	colorlinks = true,
    	linkcolor = blue,
    	anchorcolor = blue,
    	citecolor = blue,
    	filecolor = blue,
    	urlcolor = blue
    }

\usepackage[utf8]{inputenc} 
\usepackage[T1]{fontenc}    
\usepackage{hyperref}       
\usepackage{url}            
\usepackage{booktabs}       
\usepackage{amsfonts}       
\usepackage{nicefrac}       
\usepackage{microtype}      
\usepackage{xcolor}         

\usepackage{xcolor}
\usepackage{float}
\usepackage{amssymb}
\usepackage{amstext}
\usepackage{amsmath}
\usepackage{amscd}
\usepackage{amsthm}
\usepackage{amsfonts}
\usepackage{enumerate}
\usepackage{graphicx}
\usepackage{latexsym}
\usepackage{mathrsfs}
\usepackage{mathtools}
\usepackage{cases}
\usepackage{verbatim}
\usepackage{bm}
\usepackage{tikz}
\usepackage{subcaption}
\usepackage{tabularx}
\usepackage{caption}
\usepackage{adjustbox}
\usepackage{authblk}
\usepackage{hyperref} 
\usepackage{cleveref}
\usepackage{tikz}
\usetikzlibrary{matrix}
\usepackage{hyperref} 
\usepackage{comment}
\usepackage{tikz-cd}
\usepackage{pb-diagram}
\usepackage[subrefformat=parens]{subcaption}
\usepackage{wrapfig}
\usepackage{booktabs}

\newcommand{ \eqdef }{   \ensuremath{\stackrel{\mbox{\upshape\tiny def.}}{=}}
}

\definecolor{darkmidnightblue}{rgb}{0.0, 0.2, 0.4}
\definecolor{aogreen}{rgb}{0.0, 0.5, 0.0}

\definecolor{brickred}{rgb}{0.8, 0.25, 0.33}

\newtheorem{setting}{Setting}
\newtheorem{assumption}{Assumption}
\newtheorem{lemma}{Lemma}
\newtheorem{theorem}{Theorem}

\newtheorem{definition}{Definition}
\newtheorem{example}{Example}

\newtheorem{remark}{Remark}

\definecolor{darkcyan}{rgb}{0.0, 0.55, 0.55}
\definecolor{MidnightBlue}{RGB}{25,25,112}
\definecolor{MidnightBlueComplementingGreen}{RGB}{25,112,25}
\definecolor{MidnightBlueComplementingPurple}{RGB}{112,25,112}
\definecolor{MidnightBlueComplementingRed}{RGB}{112,25,69}
\definecolor{WowColor}{rgb}{.75,0,.75}
\definecolor{MildlyAlarming}{rgb}{0.85,0.25,0.1}
\definecolor{SubtleColor}{rgb}{0,0,.50}
\definecolor{antiquefuchsia}{rgb}{0.57, 0.36, 0.51}
\definecolor{fashionfuchsia}{rgb}{0.96, 0.0, 0.63}
\definecolor{jade}{rgb}{0.0, 0.66, 0.42}
\definecolor{caribbeangreen}{rgb}{0.0, 0.8, 0.6}
\definecolor{aquamarine}{rgb}{0.5, 0.8, 0.85}
\definecolor{darkmidnightblue}{rgb}{0.0, 0.2, 0.4}
\definecolor{attentioncolor}{RGB}{152,90,81}
\definecolor{burgred}{RGB}{40,3,22}
\definecolor{AKGreen}{RGB}{17,123,92}
\definecolor{egyptianblue}{rgb}{0.06, 0.2, 0.65}
\definecolor{Turquoise}{RGB}{64,224,208}

\definecolor{darkjade}{RGB}{0,122,84}
\definecolor{Window1}{RGB}{92,150,31}%
    \definecolor{Window1dark}{RGB}{41,67,13}%
\definecolor{Window2}{RGB}{255,168,28}
    \definecolor{Window2dark}{RGB}{114,75,12}
\definecolor{Window3}{RGB}{255,96,33}
    \definecolor{Window3dark}{RGB}{97,36,12}
\definecolor{InputColor}{RGB}{20,255,177}
    \definecolor{InputColorlight}{RGB}{222,237,229}

\usepackage[colorinlistoftodos]{todonotes}

\NewDocumentCommand{\F}{o}{
    \IfValueT{#1}{
            \mathbb{F}_{#1}
        }
    \IfValueF{#1}{
            \mathbb{F}
        }
                    }

\NewDocumentCommand{\R}{o}{
    \IfValueT{#1}{
            \mathbb{R}^{#1}
        }
    \IfValueF{#1}{
            \mathbb{R}
        }
                    }
                    
\NewDocumentCommand{\N}{o}{
    \IfValueT{#1}{
            \mathbb{N}^{#1}
        }
    \IfValueF{#1}{
            \mathbb{N}
        }
                    }

\newcommand{\cF}{\mathcal{F}}
\newcommand{\cN}{\mathcal{N}}
\newcommand{\E}{\mathbb{E}}

\usepackage{marginnote}
\definecolor{MidnightBlue}{RGB}{25,25,112}
\definecolor{MidnightBlueComplementingGreen}{RGB}{25,112,25}
\definecolor{MidnightBlueComplementingPurple}{RGB}{112,25,112}
\definecolor{MidnightBlueComplementingRed}{RGB}{112,25,69}
\definecolor{coolblack}{rgb}{0.0, 0.18, 0.39}

\definecolor{deepjunglegreen}{rgb}{0.0, 0.29, 0.29}
\definecolor{applegreen}{rgb}{0.55, 0.71, 0.0}
\definecolor{WowColor}{rgb}{.75,0,.75}
\definecolor{MildlyAlarming}{rgb}{0.85,0.25,0.1}
\definecolor{SubtleColor}{rgb}{0,0,.50}
\definecolor{SubtleColor2}{rgb}{0.6,0.21,.50}

\definecolor{lasallegreen}{rgb}{0.03, 0.47, 0.19}
\newcounter{margincounter}

\NewDocumentCommand{\AK}{mo}{
    \IfValueF{#2}{
                        {{\scriptsize
                            \textcolor{deepjunglegreen}{
                            \hfill\\
                                \textbf{A:}
                                \textit{{#1}}
                            \hfill\\
                            }
                        }}
        }
    \IfValueT{#2}{
                        \marginnote{{\scriptsize
                            \textcolor{deepjunglegreen}{ 
                            \textbf{A:}
                            \textit{{#1}}
                            }
                        }}
        }
                    }

\newcommand\numberthis{\addtocounter{equation}{1}\tag{\theequation}}

\usepackage{cleveref}
\newcounter{termcounter}
\renewcommand{\thetermcounter}{\Roman{termcounter}}
\crefname{term}{term}{terms}
\creflabelformat{term}{#2\textup{(#1)}#3}

\makeatletter
\def\term{\@ifnextchar[\term@optarg\term@noarg}
\def\term@optarg[#1]#2{%
  \textup{#1}%
  \def\@currentlabel{#1}%
  \def\cref@currentlabel{[][2147483647][]#1}%
  \cref@label[term]{#2}}
\def\term@noarg#1{%
  \refstepcounter{termcounter}%
  \textup{(\thetermcounter)}%
  \cref@label[term]{#1}}
\makeatother

\title{Sharp Generalization Bounds for Foundation Models with Asymmetric Randomized Low-Rank Adapters}

\renewcommand\numberthis{\addtocounter{equation}{1}\tag{\theequation}}

\renewcommand{\arraystretch}{1.3}  

\author{Anastasis Kratsios\thanks{Corresponding Author: kratsioa@mcmaster.ca}, Tin Sum Cheng, Aurelien Lucchi, Haitz S\'{a}ez de Oc\'{a}riz Borde}

\usepackage[most]{tcolorbox}
\usepackage{etoolbox} 
\usepackage{xcolor}
\definecolor{faintgray}{RGB}{245,245,245}     
\definecolor{faintborder}{RGB}{230,230,230}   
\definecolor{lightblack}{gray}{0.4}           

\newcounter{question}
\newtcolorbox[auto counter, use counter=question]{question}[1][]{
  enhanced,
  colback=faintgray,
  colframe=faintborder,
  boxrule=0.2pt,
  arc=2mm,
  title=\textcolor{lightblack}{\textbf{Question~\thequestion}},
  fonttitle=\bfseries,
  before upper={\centering\itshape},
  after title={\vspace{0.5ex}},
  boxsep=4pt,
  left=6pt,
  right=6pt,
  top=4pt,
  bottom=4pt,
  #1
}

\begin{document}

\maketitle

\begin{abstract}
Low-Rank Adaptation (LoRA) has emerged as a widely adopted parameter-efficient fine-tuning (PEFT) technique for foundation models. Recent work has highlighted an inherent asymmetry in the initialization of LoRA's low-rank factors, which has been present since its inception and was presumably derived experimentally. This paper focuses on providing a comprehensive theoretical characterization of asymmetric LoRA with frozen random factors. First, while existing research provides upper‐bound generalization guarantees based on averages over multiple experiments, the behaviour of a single fine-tuning run with specific random factors remains an open question. We address this by investigating the concentration of the typical LoRA generalization gap around its mean. Our main upper bound reveals a sample complexity of $\tilde{\mathcal{O}}\left(\frac{\sqrt{r}}{\sqrt{N}}\right)$ with high probability for rank $r$ LoRAs trained on $N$ samples. Additionally, we also determine the fundamental limits in terms of sample efficiency, establishing a matching lower bound of $\mathcal{O}\left(\frac{1}{\sqrt{N}}\right)$. By more closely reflecting the practical scenario of a single fine-tuning run, our findings offer crucial insights into the reliability and practicality of asymmetric LoRA.
\end{abstract}

\section{Introduction}
\label{s:Introduction}


Recent years have witnessed a shift from specialized models to large foundation models capable of performing a plethora of tasks, particularly in language~\citep{touvron2023,openai2023,bai2023qwentechnicalreport,qwen2025qwen25technicalreport,deepseekai2025deepseekv3technicalreport}.  This paradigm shift has led to the development of increasingly large models, often amounting to billions of weight parameters, trained on massive compute clusters by major high-tech companies.  Downstream consumers and researchers frequently seek to deploy these models by fine-tuning them on proprietary data for specific tasks or by personalizing them.  However, a significant disparity often exists between the computational resources used to train these massive models and the resources available to downstream users.  

\paragraph{Low-Rank Adaptation} This discrepancy has driven the widespread adoption of parameter-efficient fine-tuning (PEFT) methods~\citep{he2022towards, pfeiffer-etal-2020-adapterhub, ding2022deltatuningcomprehensivestudy, yu2022vpetl, han2024parameterefficientfinetuninglargemodels}.  Among these techniques, Low-Rank Adaptation (LoRA) \citep{hu2022LoRA_OG} has become particularly widespread, arguably due to its effectiveness and simplicity.
LoRA creates low-rank factors by parameterizing perturbations to pre-existing weight matrices $W$ of the form $\Delta W=BA$ where both $W$ and $\Delta W$ are, say, $d_{out}\times d_{in}$ matrices and the low-rank factor matrices $B$ and $A$ are $d_{out}\times r$ and $r\times d_{in}$ respectively with $1\le r\ll d_{in},d_{out}$ (thus leading to parametric efficiency). After fine-tuning, it suffices to substitute the original matrix $W$ with its updated variant:
\begin{equation}
\label{eq:LoRA_def}
    W' = W + \Delta W \eqdef W + BA.
\end{equation}
This technique achieves comparable performance to full fine-tuning by recognizing and exploiting the observation that only a small fraction of a large pre-trained model's parameters must be adjusted to adapt to a new task; at a fraction of its computational cost. Such a formulation relies on the assumption that the distribution shift between the pre-training and post-training dataset is not large, which is empirically supported in the context of language. Fine-tuning only a small subset of parameters also helps avoid issues such as catastrophic forgetting~\citep{French1999CatastrophicFI}, ensuring the new model retains the extensive knowledge gained during pre-training by limiting its divergence from the base model.


\paragraph{Asymmetry in Low-Rank Adapters} Among the variety of LoRA techniques that have recently been proposed in the literature (see Section~\ref{s:Introduction_RelatedWorks}), we focus on the lightweight paradigm introduced by~\cite{zhu2024asymmetry}. There, the authors demonstrate that LoRA components are inherently asymmetric. This asymmetry is evident in the typical LoRA initialization, where, for full LoRA (tuning both $A$ and $B$), $A$ is commonly initialized to a random Gaussian matrix (from now on referred to as a \textit{random factor}), while $B$ is initialized to zero~\citep{hu2022LoRA_OG}. This empirically derived procedure aligns with the distinct roles of the $A$ (\textit{feature projector}) and $B$ (\textit{feature extractor}) matrices: $A$ extracts features from the input, and $B$ uses these features to create the desired output. If the goal of adaptation is to approximate a desired output, selectively projecting away a portion of the input feature space might be less damaging than projecting away a significant portion of the output space, especially when the LoRA rank $r$ is much smaller than the output dimension $d_{out}$. This is particularly true if the input features contain redundant information or exhibit a low-rank distribution. We further study this asymmetric LoRA paradigm since it is both simple and efficient (potentially leading to a 2x trainable parameter reduction when $d_{in}=d_{out}$), see Figure~\ref{fig:asymmetry_schematic}.

\begin{figure}[H]
    \centering
    \vspace{-1em}
    \includegraphics[width=0.5\textwidth]{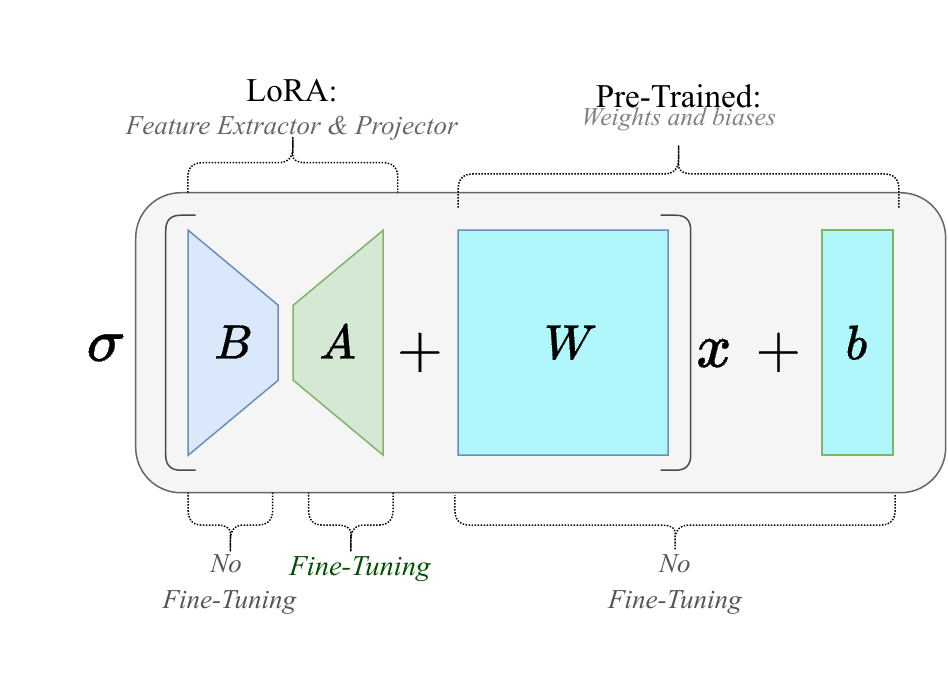}
    \vspace{-2em}
    \caption{Asymmetric low-rank adapter schematic for simple linear layer and a non-linearity $\sigma$. Note that in practice, in Large Language Models (LLMs) we fine-tune the attention matrices instead.}
    \label{fig:asymmetry_schematic}
\end{figure}

\paragraph{Open Theoretical Questions} In this work, we focus on \textit{providing a full theoretical characterization of the asymmetric LoRA case with frozen random factors}. While~\cite{zhu2024asymmetry} discusses valuable intuition and introduces initial theoretical groundwork for the asymmetry phenomenon, several important questions with practical implications for LoRA users remain open. In particular, the original paper derives only an upper bound on the average generalization gap and its conclusions only hold over several, perhaps thousands, of experiments and instantiations of random LoRA factors: they are derived for the \textit{average} generalization gap, calculated over random draws of the data and the random initialization of the $A$ matrix. Hence, the first question we aim to answer is:

\begin{equation}
\tag{Q1}
\textit{How rapidly does the \textbf{typical} LoRA generalization gap concentrate around its mean?}
\label{Q:Concentration}    
\end{equation}

Simply put, a positive answer would imply that the generalization gap trends hold whenever a practitioner runs a \textit{single} fine-tuning experiment with a specific model and a \textit{single} randomly chosen factor. Therefore, understanding the \textit{concentration} of the generalization performance for a \textit{specific} model around the reported average is crucial for assessing the robustness and predictability of the asymmetric LoRA method in real-world scenarios. A natural progression of inquiry leads to:
\begin{equation}
\tag{Q2}
\textit{What are the fundamental limits on the sample efficiency of asymmetric LoRA?}
\label{Q:LB}
\end{equation}

After establishing a better understanding of the generalization gap (as in \ref{Q:Concentration}), the next critical question is how much data the model needs to achieve that generalization: even strong generalization abilities are less useful if they require large amounts of training data. Thus, \ref{Q:LB} seeks to determine the fundamental limits on how efficiently asymmetric LoRA can learn in practice.

\paragraph{Contributions} We directly address the two critical open questions in the learning theory of LoRA:

\begin{itemize}
    \item Regarding \ref{Q:Concentration}, our main upper bound on the sample complexity of the asymmetric LoRA paradigm of~\cite{zhu2024asymmetry}, derived in Theorem~\ref{thrm:Main_UB}, shows that rank $r$ LoRAs achieve a sample complexity of $\tilde{\mathcal{O}}\left(\frac{\sqrt{r}}{\sqrt{N}}\right)$ (ignoring logarithmic terms), with high probability, from $N$ training samples. We deduce that although an increased rank $r$ may indeed increase expressiveness, as shown in~\cite{zeng2024the}, it does so at the cost of a wider generalization gap. Our worst-case theoretical guarantees, which hold uniformly over all fine-tuning tasks and all training algorithms, are verified experimentally in Section~\ref{s:Experiments}.
    \item We then turn our attention to~\ref{Q:LB}, by affirming in Theorem~\ref{thrm:LBs} that the typical sample complexity of $\mathcal{O}\left(\frac{1}{\sqrt{N}}\right)$ is indeed optimal in terms of \textit{sample size}. We do this by constructing a neural network with a random LoRA factor and a family of data-generating distributions achieving the upper bound in our first result; thus \textit{yielding a matching lower bound}.
\end{itemize}

We positively answer both~\ref{Q:Concentration} and~\ref{Q:LB}. The theoretical analysis presented in this paper relies on a new combination of techniques which are non-standard in learning theory.  These include techniques known in neural network approximation and random matrix theory for entirely different purposes.

\paragraph{Technical Contributions}
Our upper bound is derived from the recent tools in constructive approximation; namely, from the theory of Lipschitz widths~\citep{petrova2023lipschitz}, which, in~\cite{PetrovaWojtaszczyk_LimitationsMLPs_JMLR_2023}, derive sharp estimates on the local-Lipschitz regularity of the map sending parameters of a neural network to the network that those parameters realize in the space of continuous functions.  That Lipschitz constant is then estimated, with high probability, using random matrix theory results~\citep{gordon1992majorization}.  These estimates on the Lipschitz constant of the LoRA-factors-to-neural network map, together with classical estimates covering numbers on high-dimensional cubes (e.g.\ in~\cite{lorentz1996constructive}) and Dudley entropy-integral estimates, allow us to obtain our upper-bound on LoRA generalization. For the proof of our lower bound we utilize some known constructions in the approximation theory for neural networks where we emulate identity blocks with (randomized) MLP layers.  Upon doing so, we appeal to recent anti-concentration inequalities of~\cite{LittlewoodOffordProblem_VershyninRudelson_2008} of Littlewood–Offord-type, typical in modern random matrix theory; e.g.~\cite{tao2009inverse,tao2010sharp} for random variables on $[0,1]$, showing that our bound is tight. In either case, we integrate techniques from (universal) approximation theory of neural networks with learning theoretic and random matrix theoretic tools to derive our results.  This shows how not only the results, but also that techniques previously isolated within approximation theory can have learning theoretic applications.

\section{Related Works}
\label{s:Introduction_RelatedWorks}
\paragraph{A Zoo of LoRAs} Following its introduction, numerous LoRA variants have emerged, often aiming to further reduce computational overhead. Quantization, for example, offers a way to lower memory consumption both during training~\citep{gholami2021survey_quantization,Dettmers2023QLoRAEF,guo2024lqLoRA} and after~\citep{yadav2023compeft}. The number of trainable parameters can also be reduced through adaptive rank allocation~\citep{zhang2023adaLoRA}. Ideas around weight or projection reuse~\citep{frankle2018lottery,Ramanujan_2020_random_iccv} have further inspired strategies to decrease trainable LoRA parameters, such as learning diagonal rescaling of frozen random $B$ and $A$ matrices (VeRA)~\citep{kopiczko2024vera}, deriving $B$ and $A$ from the SVD of the pre-trained $W_0$ and optimizing a smaller matrix in the resulting space (SVDiff)~\citep{han2023svdiff}, learning a linear combination of fixed random matrices (NOLA)~\citep{koohpayegani2023nola}, and fine-tuning with orthogonal matrices (BOFT)~\citep{liu2024boft}. Furthermore, LoRA's applicability has recently expanded beyond classical LLM post-training and language. For example, it has been employed in the context of vision language models~\citep{li2023graphadapter} and vision Transformers~\citep{dong2023efficient}, image generative modeling for fast Stable Diffusion fine-tuning and personalization~\citep{rombach2022highresolutionimagesynthesislatent,gal2022image,ruiz2022dreambooth,roich2022pivotal}, for score distillation~\citep{wang2023prolificdreamer} (although more principled LoRA-free methods have also emerged recently~\citep{lukoianov2024score}), fine-tuning base models into reasoning models using reinforcement learning~\citep{wang2025tinatinyreasoningmodels}, and even in the development of new adapters for Graph Neural Networks and Graph Transformers~\citep{papageorgiou2025graph}.

\paragraph{Previous Theoretical Works} In terms of theoretical results, the \textit{approximation} properties of standard LoRAs have only recently come into focus, as seen in~\cite{zeng2024the}. In comparison, the statistical properties of LoRA are better understood, with PAC-Bayesian guarantees recently derived in~\cite{pmlr-v235-lotfi24a,liu2023pactuning,lei2024fast}, and guarantees in the infinite-width surrogate (NTK) limit, as detailed in~\cite{malladi2023kernel,jang2024LoRA}. Of particular interest for this paper is prior research indicating that freezing the $A$ matrix in standard (or vanilla) LoRA does not significantly affect performance~\citep{zhang2023LoRAfa}. Intriguingly, although almost all recent works initialize or freeze these two matrices asymmetrically, a rigorous investigation into the implications of this asymmetry in low-rank adaptation has only recently garnered theoretical attention. Finally, in spite of the fact that NTK-based analyses~\citep{malladi2023kernel,jang2024LoRA} provide insight into the training dynamics and loss landscape of LoRA models, their conclusions are limited to the asymptotic (infinite-width) setting: they cannot necessarily be transferred to the real-world finite-width scenarios of interest in this work.

\section{Main Statistical Guarantees}
\label{s:Main_Results}
\vspace{-.5em}
\paragraph{Setup}
We consider the generalization capacity of foundation models with low-rank randomized factors.  We consider the (random) class $\mathcal{F}$ of functions $f:\mathbb{R}^d\to \mathbb{R}^D$ admitting the representation:
\begin{equation}
\label{eq:LoRA_intro}
\setlength{\jot}{2pt} 
\begin{aligned}
        f(x)
    & \eqdef 
        (W^{(T+1)}+
        \underbrace{
            B^{(T+1)}A^{(T+1)}
        }_{\Delta W^{(T+1)}: \text{ LoRA Perturbation}}
        )x^{(T+1)} + b^{(T+1)}
\\
        x^{(t+1)}
    & \eqdef 
        \sigma\bullet 
        \big(
            (W^{(t)}+
            \underbrace{
                B^{(t)}A^{(t)}
            }_{\Delta W^{(t)}: \text{ LoRA Perturbation}}
            )x^{(t)} + b^{(t)}
        \big)
\\
    x^{(1)} & \eqdef x \in\R^d
\end{aligned}
\end{equation}
where $W^{(1)},\dots,W^{(T+1)}$ are pre-trained $d_{t+1}\times d_t$-dimensional weight matrices and $b^{(1)},\dots,b^{(T+1)}$ are pre-trained biases $d_{t+1}$-dimensional, $B^{(1)},\dots,B^{(T+1)}$ are $d_{t+1}\times r$ dimensional random (non-trainable) random LoRA factors, and $A^{(1)},\dots,A^{(T+1)}$ are $r\times d_t$ dimensions \textit{trainable LoRA factors}; here $d_1=d$, $d_2=d_3=\cdots=d_{T}=W$, and $d_{T+1}=D$ and $t=1,\dots,T$.
\vspace{-.3em}
\begin{remark}[Symmetry in Upper-Bound]
Our main upper-bound (Theorem~\ref{thrm:Main_UB}) remains valid if the $B$ factors are trainable and the $A$ factors are instead randomized, as in~\cite{zhu2024asymmetry}.  
\end{remark}
\vspace{-.3em}
We consider a $1$-Lipschitz loss function $\ell:\mathbb{R}^D\times \mathbb{R}^D\to [0,1]$ and i.i.d.\ training data $\{(X_n,Y_n)\}_{n=1}^N$ (independent of the \textit{random} LoRA factors $\{B^{(t)}\}_{t=1}^{(T+1)}$) drawn from a data-generating distribution $\mathbb{P}$ on $\mathbb{R}^d\times \mathbb{R}$.
Our objective is to compute a uniform generalization bound between the empirical risk $\mathcal{R}^N(f)$ and the true risk $\mathcal{R}(f)$ for any (random) learner $f$ in the (random) class $\mathcal{F}$; defined by
\[
        \mathcal{R}^N(f) \eqdef \frac1{N}\,\sum_{n=1}^N\, \ell(f(X_n),Y_n)
    \mbox{ and }
        \mathcal{R}(f) \eqdef \mathbb{E}_{(X,Y)\sim \mathbb{P}}\big[ \ell(f(X),Y)\big]
.
\]
We emphasize that, unlike classical learning theoretic results, our function class $\mathcal{F}$ is random.
We want to derive a high probability bound on the worst-case, uniformly over all data-generating distributions and all training algorithms mapping training data to a learner.  This randomized generalization gap is
\begin{equation}
\label{eq:rand_gen}
\smash{
        \mathbf{G}
    \eqdef
        \operatorname{sup}_{f\in \mathcal{F}}\,
            \big|
                \mathcal{R}(f)
                -
                \mathcal{R}^N(f)
            \big|
.
}
\end{equation}
We highlight that, unlike classical PAC-learning guarantees, the true risk is a random object, as it depends on the randomness in the LoRA factors, rather than being deterministic.

\vspace{-1em}
\paragraph{Preliminaries and Notation}
Before stating our main results, we must first rigorously define our randomized learners with random LoRA factors expressed in~\eqref{eq:LoRA_intro} at a high level.  
\paragraph{Parametric Perturbations of Multi-Layer Perceptrons}
Fix dimensions $d,D\in \mathbb{N}_+$, depth and a width parameters $T,W\in \mathbb{N}_+$, respectively, and an MLP $\hat{f}:\mathbb{R}^d\to \mathbb{R}^D$ with representation
\begin{equation}
\label{eq:MLP_Representation}
\begin{aligned}
        \hat{f}(x|\theta)
    & \eqdef 
        W^{(T+1)}x^{(T+1)} + b^{(T+1)}
\\
        x^{(t+1)}
    & \eqdef 
        \sigma\bullet (W^{(t)}x^{(t)} + b^{(t)}) \qquad \mbox{for } t=1,\dots,T
\\
    & x^{(1)}\eqdef x
\end{aligned}
\end{equation}
where $\theta = ((W^{(t)},b^{(t)}))_{t=1}^{(T+1)}\in \mathbb{R}^p$ where $p=(T-1)\, W(W+1) + (d+1)W + (D+1)W$, $W^{(t)}\in \mathbb{R}^{d_{t+1}\times d_t}$ and $b^{(t)}\in \mathbb{R}^{d_t}$ where $d_t=W$ if $t\in \{2,\dots,T\}$, $d_1=d$, and $d_{T+1}=D$.  
\hfill\\
\noindent 
We fix a parameter $\theta_{pre}\eqdef (W^{(t)},b^{(t)})_{t=1}^{(T+1)} \in \mathbb{R}^p$.  
We defined the \textit{perturbed representation map}
\begin{equation}
\label{eq:Perturbation_Map}
\begin{aligned}
     P_{T,W}^{\theta_{pre}}:\mathbb{R}^p& \rightarrow C([0,1]^d,\mathbb{R}^D)\\
        \theta
    & \mapsto 
        \hat{f}(\cdot|\theta+\theta_{pre})
.
\end{aligned}
\end{equation}
\vspace{-.5em}
\paragraph{LoRAs with Random Factors}
Fix a probability space $(\Omega,\mathcal{A},\mathbb{P})$; all random quantities will be defined thereon.
Fix random matrices $A_1,\dots,A_{T+1}$ of dimensions $W\times r$ for $l\le T$ and $D\times r$ for $l=T+1$.  
We fix maximum admissible LoRA weight sizes $M\ge 0$.
We then define the (random) \textbf{parameter-to-LoRA map}, which maps low-rank matrices $C_1,\dots,C_{T+1}$ and an event $\omega\in \Omega$ (random initialization of non-trainable LoRA parameters $A^{t}$) to the sequences of matrices:
\begin{equation}
\label{eq:LoRA}
\begin{aligned}
    \operatorname{LoRA}: 
        \Omega \times [-M,M]^{q} 
    & \rightarrow
        C([0,1]^d,\mathbb{R}^D)
    \\
        \big(\omega,(A^{(t)})_t\big)
    & \mapsto
        P_{T,W}^{\theta_{pre}}\big(
            (B^{(t)}(\omega)A^{(t)})_{t=1}^{(T+1)}
        \big)
    \\
    & =
        \hat{f}(\cdot| (B^{(t)}(\omega)A^{(t)})_{t=1}^{(T+1)} + \vartheta_{pre})
\end{aligned}
\end{equation}
where the effective number of trainable LoRA parameters is defined to be
\begin{equation}
\label{eq:effective_dimension_LoRA}
    q\eqdef r(W(T-1)+d+D)
    \in 
    \mathcal{O}(r)
.
\end{equation}

\paragraph{Notation} We now close our preliminaries section with the following table
aggregating the notation used in formulating our main result within the main body of our paper.

\begin{table}[H]
\renewcommand{\arraystretch}{0.9} 
\adjustbox{max width=\textwidth}{%
\centering
    \begin{tabular}{ll}
    \toprule
    \textbf{Symbol} & \textbf{Meaning} \\
    \midrule
    $d$      &     input dimension    \\
    $D$     &    output dimension \\
    $T$        & network depth\\
    $W$      &     pre-trained network width    \\
    $r$     &    LoRA width \\
    $p=W(TW-W+T+d+D+1)$      &     number of pre-trained parameters $\theta^\text{pre}=(W^{(t)},b^{(t)})_t$  \\
    $q=r(TW-W+d+D)$     &    number of trainable LoRA parameters $(A^{(t)})_t$ \\
    $M$ & maximum admissible values for trainable LoRA parameters $(A^{(t)})_t$\\
    $\ell$ & $L_{\ell}\ge 0$ Lipschitz classification-loss from $\mathbb{R}^D\times \mathbb{R}^D$ to $[0,1]$ \\
    \bottomrule
    \end{tabular}
}
\caption{Notation used in the main body of our manuscript.}
\label{tab:Notation}
\end{table}
\vspace{-2em}
Henceforth, we always operate in the following setting:
\begin{setting}
\label{setting}
Fix $W,T,r,d,D\in \mathbb{N}_+$ with $1\le r<W$.  We fix a continuous activation function $\sigma:\mathbb{R}\to \mathbb{R}$ which is either a bounded Lipschitz or the ReLU activation function. We assume there is a constant $M>0$ such that $A^{(t)}_{ij}<M$ for all $t,i,j$ throughout the LoRA training.
\end{setting}
\vspace{-.5em}
\subsection{Main Upper-Bound on LoRA Generalization Gap}
\label{s:Main_Results__ss:UpperBound}
\vspace{-.5em}
Our main theorem is the following randomized generalization bound.
\begin{theorem}[LoRA Sample Complexity: Upper-Bound]
\label{thrm:Main_UB}
In setting~\ref{setting}: for every failure probability $0<\delta\le 1$, the following holds with probability at least $1-\delta$
\[
        \mathbf{G} 
    \leq 
        4\min \left\{  1, 
        \sqrt{q}
        \frac{6\sqrt{A}}{\sqrt{N}}  \right\} + \sqrt{\frac{8\log (2/(1-\sqrt{1-\delta})}{N}}
\]
where $A=(cT+1)\log(2R+2R_0)$, $R=M\nu\sqrt{2r\log(2W/\epsilon)}$ and $R_0=\|\theta_\text{pre}\|_\infty$.
\end{theorem}
{
\begin{remark}[Symmetry in the Upper Bound - Randomizing $B$ vs.\ $A$]
\label{rem:Symmetry_BA}
In~\cite{zhu2024asymmetry}, the authors train $B$ and randomize-then-freeze $A$ during training. Our upper bound in Theorem~\ref{thrm:Main_UB} applies equally to this setup, and the reader may draw analogous conclusions if the roles of $A$ and $B$ are reversed.  This is uncovered within the details of our proof of Theorem~\ref{thrm:Main_UB}. 
\end{remark}
}

As a function of $N$, our main generalization bounds for \textit{random} LoRAs are generally tight and cannot be improved. Our second supporting result confirms this by exhibiting a data-generating distribution for a classification problem where our bound is sharp \textit{over the random LoRA class}, assuming the random LoRA factors are standard Gaussian matrices as in~\cite{zhu2024asymmetry}.

\subsection{Matching Lower-Bound on LoRA Generalization Gap}
\label{s:Main_Results__ss:LowerBound}
The following is a simple \textit{sufficient condition} which is enough to demonstrate the optimality of the rate in sample complexity estimate (Theorem~\ref{thrm:Main_UB}), as a function of the same size $N$.  We underscore that the identification of necessary conditions for optimality of our rate, in the sample size $N$, is an interesting question in its own right; yet it is tangential to our paper's primary objective.
\begin{assumption}[A Sufficient Condition of Optimality]
\label{ass:BreakingDistribution}
Let $d=1$, $\mathbb{P}=\mathbb{P}_X\otimes \delta_0$, such that the ``sampling distribution'' $\mathbb{P}_X$ is supported on $[0,1]$ such that, if $X\sim \mathbb{P}_X$, has mean and variance $1/2$. 
\end{assumption}
Assumption~\ref{ass:BreakingDistribution} is non-vacuous, with a simple example establishing existence being a standard fair Bernoulli trial.  Though, naturally, there is a plethora of more involved examples one can easily construct by dabbling.
\begin{example}[Bernoulli Trial]
\label{ex:breaking} 
Any fair Bernoulli trial $\mathbb{P}_X(X=1)=\mathbb{P}_X(X=0)=1/2$ satisfies Assumption~\ref{ass:BreakingDistribution}.  Moreover, $\xi\eqdef 2X -2$ is a Rademacher random variable (random sign).
\end{example}
\begin{theorem}[{Optimality of the Rate In The Sample Size $\Theta(N)$}]
\label{thrm:LBs}
Let $d_0(=d)=d_{T+1}\eqdef 1$.
Consider the loss function $\ell:\mathbb{R}^2\to [0,\infty)$ defined by $\ell(\hat{y},y)\eqdef |\hat{y}-y|$ and suppose that each MLP in the class $\mathcal{F}$ represented by~\eqref{eq:MLP_Representation} satisfies the following minimum width requirement on its hidden layers
\begin{equation}
\label{eq:dimension_condition}
        \eta_{\star}
    \eqdef 
        \min_{t=1,\dots,T}
            \sqrt{d_{t+1}}
            -
            \sqrt{r}
        >
            0
.
\end{equation}
Given any data-generating probability distribution $\mathbb{P}$ satisfying Assumption~\ref{ass:BreakingDistribution},  if the entries of the random LoRA factors $\{B^{(t)}\}_{t=0}^T$ are i.i.d.\ standard normal then: there is an absolute constant $c>0$ such that for every $
2(T+1)e^{-c\eta^2}<\delta < 2(T+1) 
$ there is an $M>0$ (required for the LoRA parameterization~\eqref{eq:LoRA}) large enough such that 
\vspace{-1em}
\[
    \mathbb{P}\biggl(
            \sup_{f\in \mathcal{F}}\, \big|\mathcal{R}(f) - \mathcal{R}^N(f)\big|
        >
            \frac{1}{N}
    \biggr)
    \ge 
        \overbrace{
            \Big(
                1- \frac{\delta}{2}
            \Big)
        }^{
        \overset{
            \text{Prob. Sampling}
            }{\text{\tiny Good LoRA Weights}}
        }
        \,
        \overbrace{
            \Big(
                1- \Theta\Big(\frac{1}{\sqrt{N}}
                \Big)
            \Big)
        }^{
        \overset{
            \text{Sample Complexity}
            }{\text{\tiny (in $N$)}}
        }
.
\]
Moreover, ``large enough'' value of $M$ means that $M\in \Theta\Big(
1/\big(\sqrt{d_1}-\sqrt{d_{T+1}} - \ln(2T/\delta)\big)
\Big)$.
\end{theorem}

{
\begin{remark}[The Asymmetry in the Lower Bound - Randomizing $B$ vs.\ $A$]
\label{rem:Assymetries_BA}
In contrast to Remark~\ref{rem:Symmetry_BA}, our matching lower bound in Theorem~\ref{thrm:LBs} critically relies on the randomization of the extractor $B$. This is because $B$ tends to be invertible with high probability, by Gordon’s Theorem (see, e.g.,~\citep[Exercise 7.3.4]{vershynin2010introduction}); in essence, $\eta^{\star}$ in~\eqref{eq:dimension_condition} \textit{can} be bounded away from $0$.  
\hfill\\
In contrast, randomizing the feature projector $A$ yields only a vacuous lower bound on its smallest singular value via Gordon’s Theorem, preventing any conclusion about training $B$ ``cancelling out'' $A$; in essence, $\eta^{\star}$ in~\eqref{eq:dimension_condition} \textit{cannot} be bounded away from $0$.
\end{remark}
}
We now provide an overview of the core ideas behind the proof of our upper bound.  All additional details are relegated to our paper's appendix and the proof of the lower bound.
\vspace{-.5em}
\subsection{Implications for LoRA End-Users}
\label{s:Practical_Implications}
\vspace{-.5em}
It is precisely Lemma~\ref{lem:Lipschitz_bound_LoRA} that suggests why LoRAs with a single random factor generalize better than LoRAs with both trainable factors.    Namely, if both $A$ and $B$, in~\eqref{eq:LoRA_def}, are trainable then the Lipschitz constant of the \textit{parameter-to-LoRA} map is defined as the composition of the map 
\begin{equation}
\label{eq:why_yesyes}
\smash{
    B\mapsto AB
}
\end{equation}
and the map sends a set of neural network parameters to its realization.  If only $B$ is trainable, then the map~\eqref{eq:why_yesyes} is {\color{green}{\textbf{linear}}} meaning that its derivative is of a \textit{constant} order on the scale of the largest singular value of $A$.  Now, if \textit{both} and $A$ and $B$ are trainable then the parameter-to-LoRA map is instead pre-composed by
\begin{equation}
\label{eq:why_nono}
\smash{
    (A,B)\mapsto AB
}
\end{equation}
which is {\color{red}{\textbf{quadratic}}}.  Thus, its derivative \textit{grows linearly} and is, in particular, unbounded.  Consequently, the resulting Lipschitz constant of the parameter-to-LoRA map~\eqref{eq:LoRA_def} would be significantly larger.  Since every downstream estimate in Lemmata~\ref{lem:Covering_FunctionSpace}-\ref{lem:Bound_t:CondLipschitz} scales as a function of this constant, then training both factors seems to yield significantly larger covering numbers of the class of (random) LoRA learners and consequently larger (random) generalization bounds. 

This aligns with the information-theoretic bounds derived in~\citep{zhu2024asymmetry} (Lemma 4.5) and ultimately reinforces the practical guideline that, given a parameter budget, freezing $A$ and doubling the rank while fine-tuning only $B$ is preferable to distributing trainable parameters across lower-rank $B$ and $A$ matrices. Note that while this argument is technically symmetric in $A$ and $B$ (both in terms of the Lipschitz constant reasoning and the results in~\citep{zhu2024asymmetry}), the parameter budget is typically allocated to the $B$ matrix due to its role as a feature projector.

\vspace{-1em}
\section{Explanation via Proof Strategy for Upper-Bound}
\label{s:Main_Results__ss:Explanation}
We recall that for any $\varepsilon>0$ the covering number of a, defined rigorously in~\ref{s:Proof_Main_UB__ss:AuxiliaryDefinitions}, of a set if a metric space is simply the number of balls of radius $\varepsilon$ required to cover every point in that set.   
We begin by first quantifying the generalization bounds of MLPs induced from their parameter spaces.  Though this may yield loose bounds in general, due to inherent symmetries of MLPs\footnote{E.g. $(-1,1)^{\top}\operatorname{ReLU}((1,-1)x)=(1,-1)^{\top}\operatorname{ReLU}((-1,1)x)$.} this seems to be the most natural approach for LoRAs, which operate directly on the MLP parameter space.
\vspace{-.5em}
\paragraph{Step 1 - Probabilistic Bounds on Lipschitz Regularity of Randomized LoRA Class}

We begin by controlling the \textit{random} Lipschitz constant of our \textit{random} parameter-to-LoRA map, defined in~\ref{eq:LoRA}.  

\begin{lemma}[{High Probability Bound on Lipschitz Regularity of Parameter-to-LoRA Map~\eqref{eq:LoRA}}]
\label{lem:Lipschitz_bound_LoRA:main}
In Setting~\ref{setting}.
Then with probability at least $1-\epsilon$ (over the random initialization on $(B^{(t)})_t$), the Lipschitz constant $L_{LoRA}^{\theta_{pre}}$ of the (random) LoRA map in~\eqref{eq:LoRA} is at-most 
\vspace{-.5em}
\[
        L_{LoRA}^{\theta_{pre}}
    \le 
        2^{c_2T}\big(
        \underbrace{M\nu\sqrt{2r\log(2W/\epsilon)}}_{R}
        +
        \underbrace{\|\theta_{pre}\|_{\infty}}_{R_0}
    \big)^{c_2T}.
\]
\vspace{-1em}
where $\nu>0$ is the random initialization scale of $B^{(t)}_{ij}\sim\mathcal{N}(0,\nu^2)$.
\end{lemma}
\paragraph{Step 2 - Probabilistic Bounds on Covering Number of Randomized LoRA Class}
Having obtained a probabilistic estimate of the Lipschitz constant for our random LoRA parameterization map, we derive a probabilistic bound on the covering number of the associated randomized function class; conditioned on the event depending on the draw of the randomized LoRA parameters:
\begin{equation}
\label{eq:conditional_control__LipschitzConstant:main}
    \mathcal{B}_{\epsilon}
\eqdef 
    \biggl\{
        \omega\in \Omega
    :
            L_{LoRA}^{\theta_{pre}}
        \le 
            2^{c_2T}
            \big(
                M\nu\sqrt{2r\log(2W/\epsilon)}
            +
                \|\theta_{pre}\|_{\infty}
            \big)^{c_2T}
    \biggr\}
\end{equation}
The normed space under consideration is typically the parameter space $\R^p$, where the finite collection of model parameters is flattened and concatenated into a single vector. Unless otherwise stated, the norm used is the $\infty$-norm.
We begin with a minor extension—requiring only a brief comment—of a classical covering estimate for balls in normed spaces; see, for instance, \cite[Proposition 15.1.3]{lorentz1996constructive}. The integer $p$ denotes the total number of parameters defining the LoRA mapping.
\begin{lemma}[Covering Number Bounds: Random LoRA Class]
\label{lem:Covering_FunctionSpace:main}
Under suitable assumptions, for each random initialization $\omega\in\Omega$ on the non-trainable parameters $(B^{(t)})_{t=1}^{T+1}$, define the LoRA function space $\mathcal{F}$ 
equipped with the metric induced by the uniform norm $\|\cdot\|_{\infty}$:
\[
    \mathcal{F} \eqdef \mathcal{F}(\omega,M)
    =
    \{ \text{LoRA}(\omega,(A^{(t)})_{t=1}^{T+1})\in C([0,1]^d,\R^D) : |A^{(t)}_{ij}| \leq M,\ \forall t,i,j \}.
\]
Then, for every $\epsilon,
\varepsilon>0$ the $\varepsilon$-covering number $\mathcal{N}(\varepsilon,\mathcal{F})$ of $\mathcal{F}$
satisfies
\begin{equation}
\label{eq:Prob_of_Covering:main}
        \mathbb{P}\Big(
            \mathcal{N}(\varepsilon,\mathcal{F})
            \leq
            \left((2R+2R_0)^{cT+1} / \varepsilon \right)^q
        \Big)
    \ge 
        \mathbb{P}(\mathcal{B}_{\epsilon})
    \ge 
        1-\epsilon
\end{equation}
where $c=c_2$, $R=M\nu\sqrt{2r\log(2W/\epsilon)}$ and $R_0=\|\theta_\text{pre}\|_\infty$ as in Lemma~\ref{lem:Bound_t:CondLipschitz}.
\end{lemma}
\paragraph{Step 3 - Generalization Bounds Conditioned on Covering Number Bounds}
We control the randomized generalization gap $\mathbf{G}$ conditioned on the event $\mathcal{B}_{\epsilon}$, as defined in~\eqref{eq:conditional_control__LipschitzConstant}.  Upon conditioning on the right realizations, which occurs with high probability, our covering number bounds give us access to Dudley's integral estimate; see e.g.~\citep[Corollary 2.2.8]{varderVaartWellner_EmpiricalProcessesBook_1996}.
\begin{lemma}[Conditional Generalization Bounds for ``Derandomized'' LoRAs]
\label{lem:Bound_t:CondLipschitz:main}
Under suitable assumptions, let  $G
    \eqdef
    \sup_{f\in \mathcal{F}}\,
        \big|
            \mathcal{R}(f)
            -
            \mathcal{R}^N(f)
        \big|$
be the generalization error. The following holds
\begin{equation}
\label{eq:generalization_conditional__GoodCovering:main}
    \mathbb{P}\Big(
                \mathbf{G} 
            \leq 
                4\min 
            \big\{  1, 
            \sqrt{q}
                6\sqrt{A}/\sqrt{N}
            \big\} + \sqrt{8\log (2/\epsilon)}/\sqrt{N}
    \big|
        \,
        \mathcal{B}_{\epsilon}
    \Big)
\ge 
    1-\epsilon
\end{equation}
where $A=(cT+1)\log(2R+2R_0)$, $R=M\nu\sqrt{2r\log(2W/\epsilon)}$ and $R_0=\|\theta_\text{pre}\|_\infty$.
\end{lemma}
Theorem~\ref{thrm:Main_UB} can be deduced from here by lower-bounding $\mathbb{P}(\mathbf{G}\le G^{\star})$ upon conditioning on $\mathcal{B}_{\varepsilon}$ (which happens with probability at-least $\varepsilon$).
\section{Experimental Validation}
\label{s:Experiments}
We now verify our results experimentally, to see if the general \textit{worst-case} trend which our theory predicts is indeed reflected in practice. We remind the reader that our results are, very much, of a \textit{worst-case} nature—both \textit{distributionally agnostic} and \textit{algorithmically agnostic}. Thus, even if we do expect to see the general pattern that as the LoRA rank $r$ grows, the generalization gap increases, we do not exactly expect the square-root rate of $\tilde{\mathcal{O}}(\sqrt{r})$ (ignoring the effect of $N$) to manifest; rather, we expect that the generalization gap grows at a \textit{slower} rate than the absolute worst-case over all algorithms and data-generating distributions.

\paragraph{CLIP-LoRA} 

CLIP is a pretrained contrastive model that aligns image and text representations in a shared embedding space, enabling zero-shot and few-shot transfer across a variety of visual classification tasks. In our experiments, we apply LoRA on CLIP for downstream image classification on standard datasets as in \cite{zanella2024low}.
The downstream classifier, or predictor $f$, is the LoRA-augmented CLIP model fine-tuned using cross-entropy loss. We observe that the generalization gap $\mathbf{G}$ generally increases with the LoRA rank $r$, suggesting that higher parameterization yields improved training performance but reduced generalization.  See Appendix~\ref{a:Experiments} for experiment details. 

\begin{figure}[h]
\centering
\includegraphics[width=0.32\textwidth]{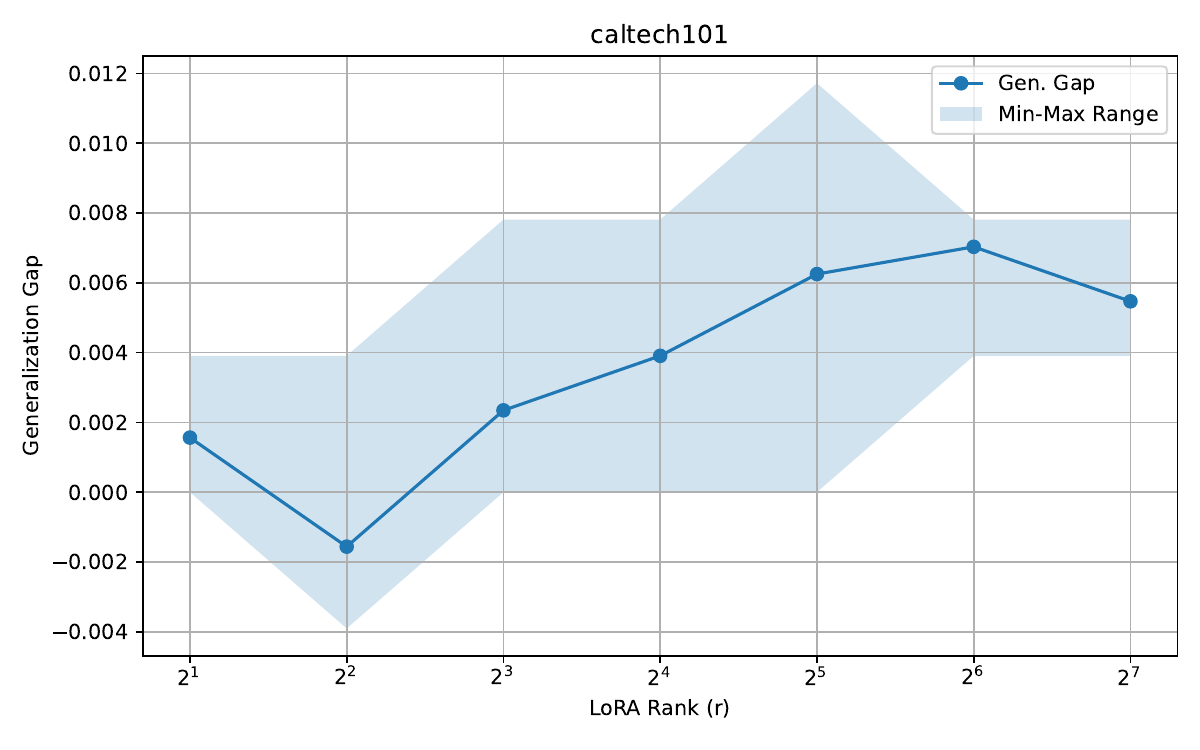}
\includegraphics[width=0.32\textwidth]{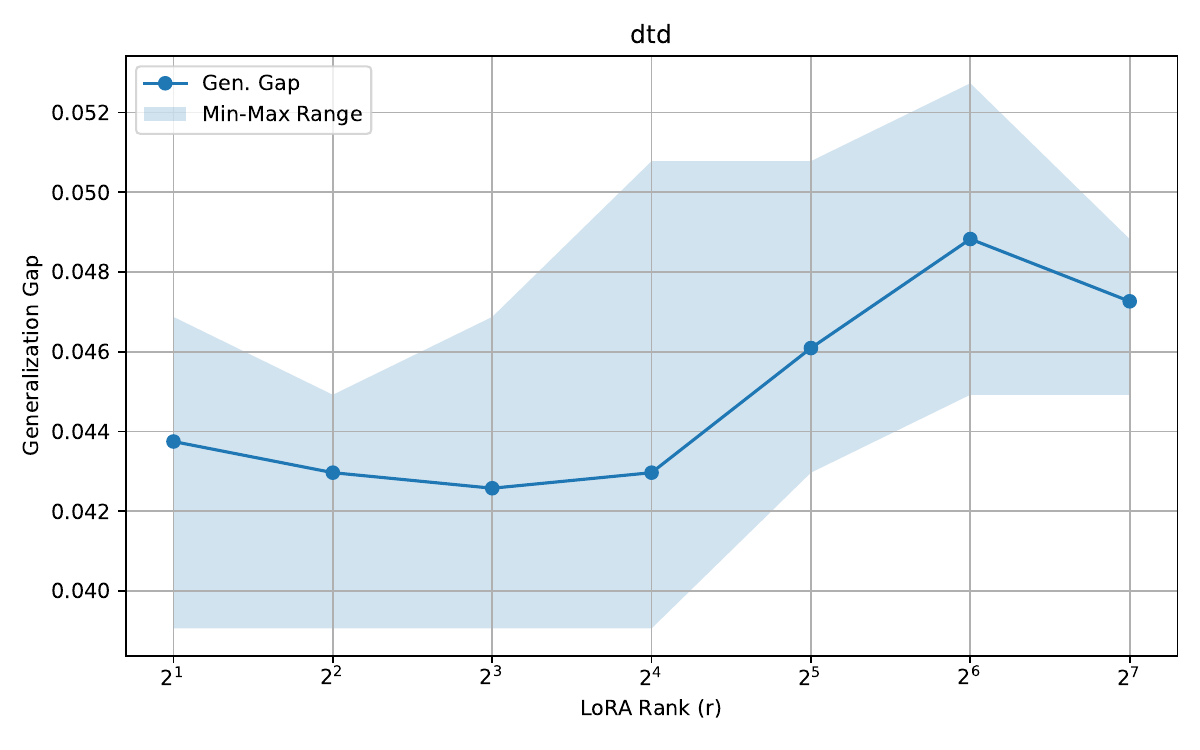}
\includegraphics[width=0.32\textwidth]{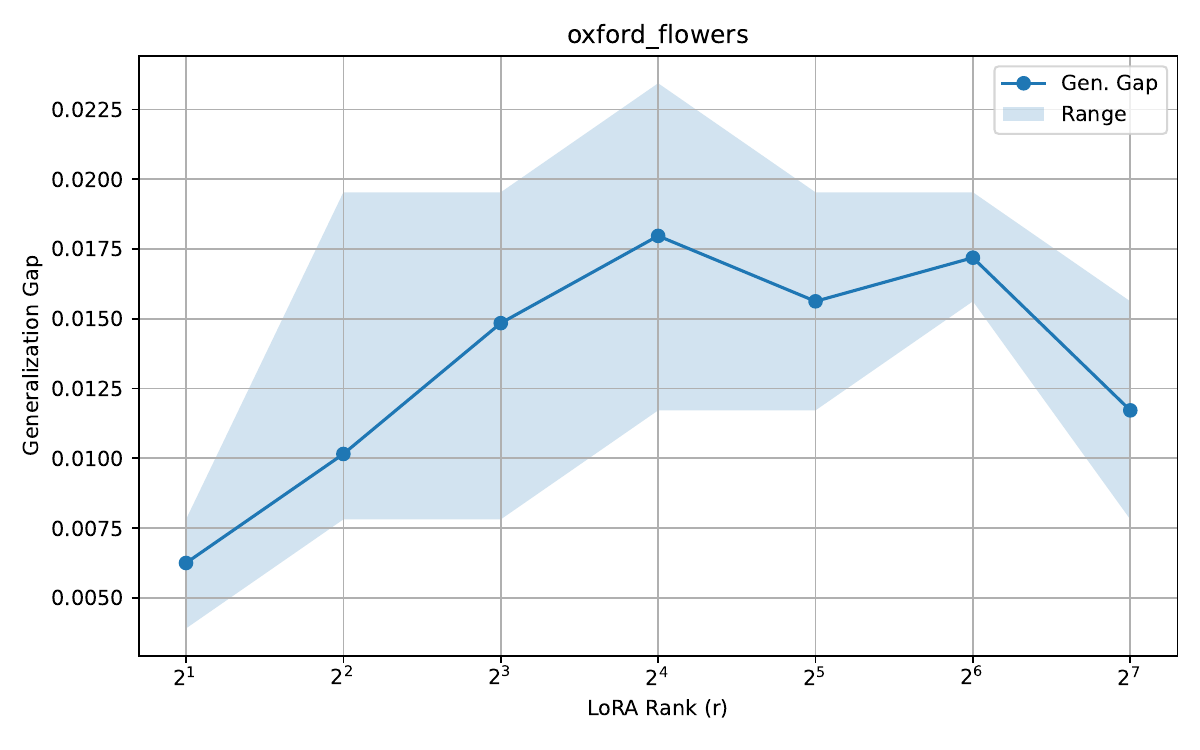}
\includegraphics[width=0.32\textwidth]{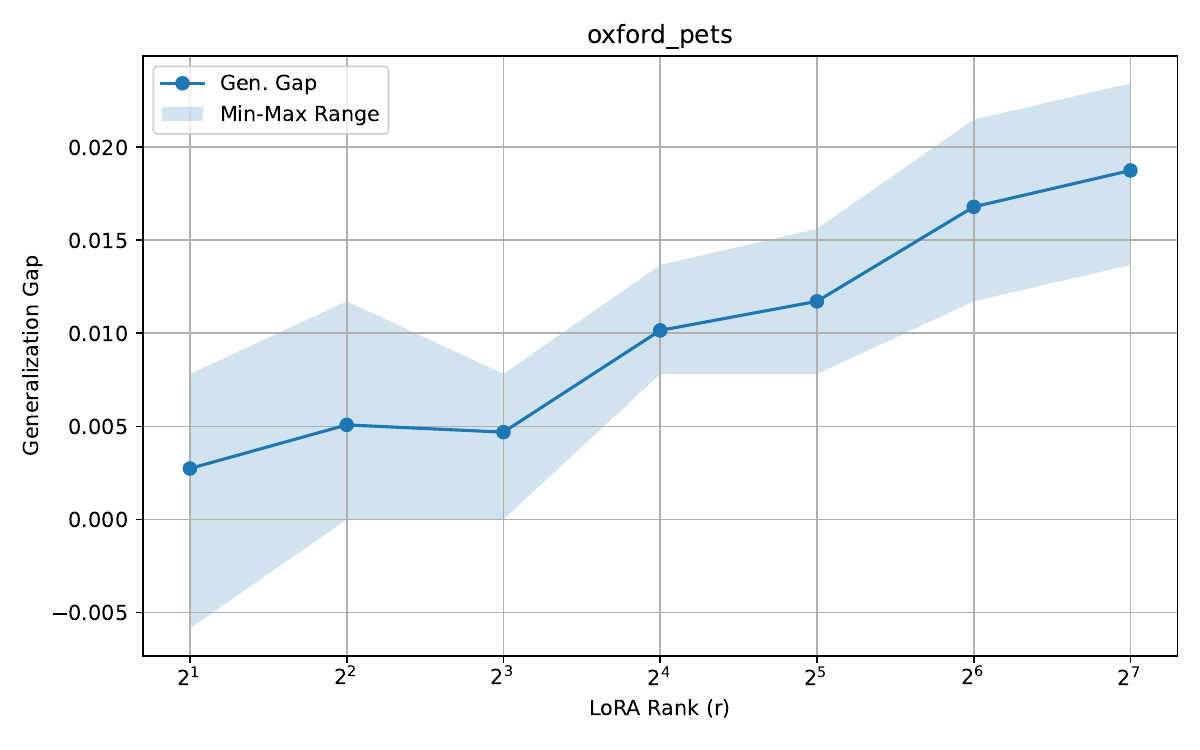}
\includegraphics[width=0.32\textwidth]{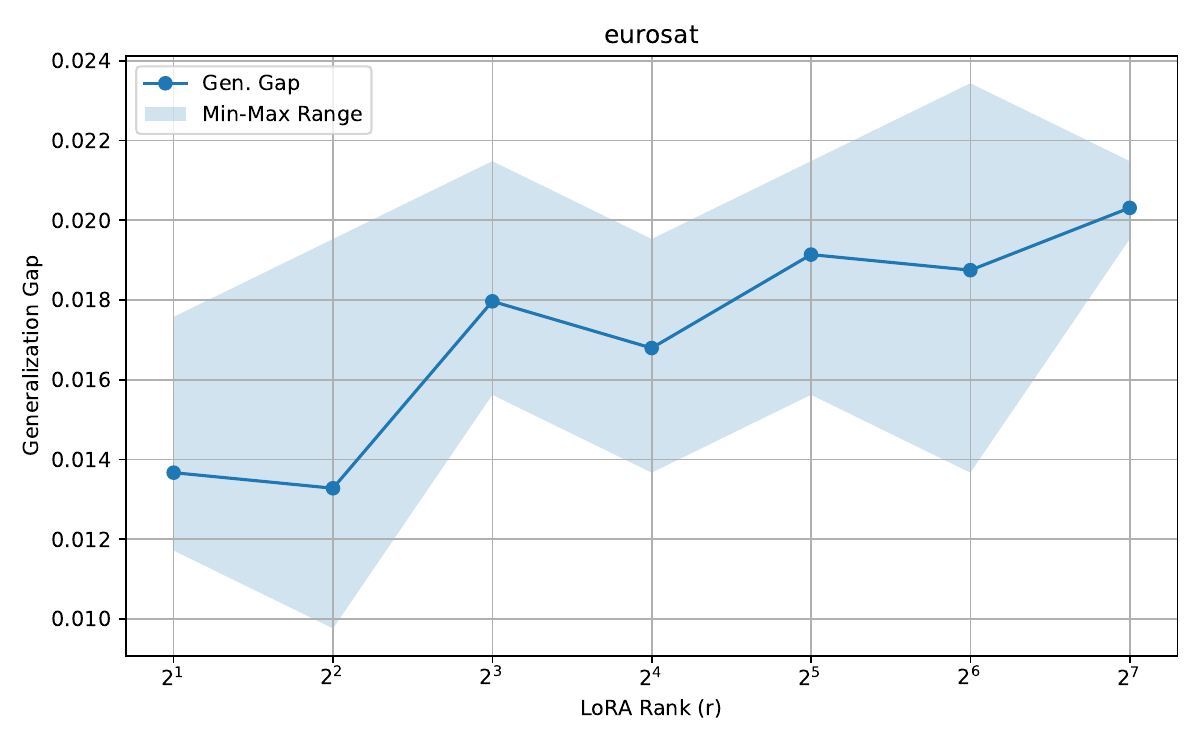}
\includegraphics[width=0.32\textwidth]{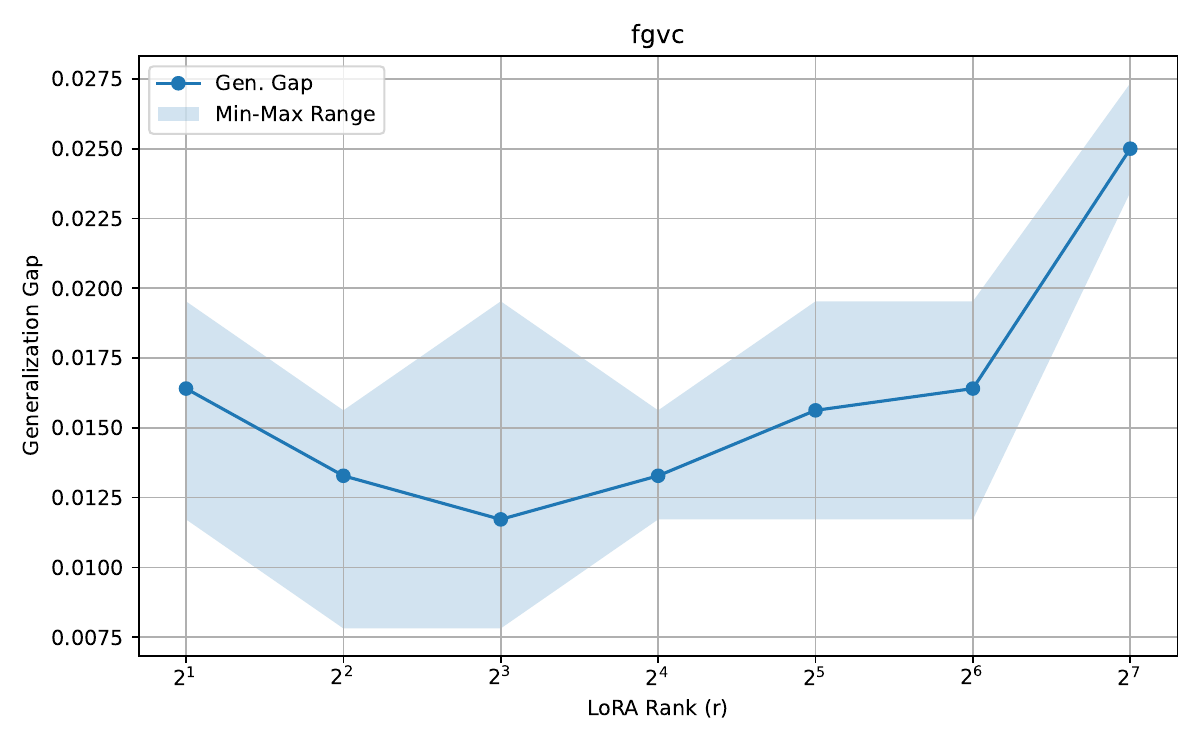}
\caption{Generalization gap on different datasets in downstream classification tasks. }
\label{fig:gap_vs_r}
\end{figure}

While this trend qualitatively aligns with theoretical predictions about overparameterization and generalization, the rate of increase in $\mathbf{G}$ with respect to $r$ is not as sharp as suggested by the theoretical upper bounds in Theorem~\ref{thrm:Main_UB}. This discrepancy is expected, as such bounds are typically agnostic to the specific optimization algorithm and dataset used, and therefore tend to be loose in practice.

\section{Conclusion}
\label{s:Conclusion}
This work provides the first high-probability sample complexity bound for LoRA fine-tuning (Theorem~\ref{thrm:Main_UB}), showing that the generalization gap scales as $\tilde{\mathcal{O}}(\sqrt{r}/\sqrt{N})$, matching prior expectations but now with rigorous tail guarantees. We complement this with a constructive lower bound (Theorem~\ref{thrm:LBs}) under a natural data distribution (Assumption~\ref{ass:BreakingDistribution}), demonstrating the rate is tight up to constants. Together, these results draw a sharp theoretical boundary for LoRA generalization with random factors, revealing an intrinsic rank–sample tradeoff. 

\section*{Acknowledgements}
A.\ Kratsios acknowledges financial support from an NSERC Discovery Grant No.\ RGPIN-2023-04482 and No.\ DGECR-2023-00230.  They both also acknowledge that resources used in preparing this research were provided, in part, by the Province of Ontario, the Government of Canada through CIFAR, and companies sponsoring the Vector Institute\footnote{\href{https://vectorinstitute.ai/partnerships/current-partners/}{https://vectorinstitute.ai/partnerships/current-partners/}}.

\bibliographystyle{plainnat}
\bibliography{ref}
\appendix

\section{Experiment Details}
\label{a:Experiments}
The experiments were based on code in \cite{zanella2024low}, which can be found at Github repository \url{https://github.com/MaxZanella/CLIP-LoRA}, which was published under the GNU AFFERO General Public license. The experiments were run on 7 NVIDIA GeForce RTX 3090 GPUs and took less than one day. The code for the experiments can be found in the supplement material

\section{\texorpdfstring{Proof of Theorem~\ref{thrm:Main_UB}}{Proof of Main Upper-Bound}}
\label{s:Proof_Main_UB}

\subsection{Auxiliary Definitions}
\label{s:Proof_Main_UB__ss:AuxiliaryDefinitions}
Before moving on to the proof of our first main result, we briefly recapitulate some standard definitions on which our analysis relies.

\begin{definition}[Covering number]
\label{def:CoveringNumber}
    Let \( (V, \|\cdot\|) \) be a normed space. 
    \begin{enumerate}[a)]
        \item 
        Let $B(\rho,v_0)\eqdef \{v\in V:\|v-v_0\|\leq \rho\}$ be the ball with radius $\rho$ and centre $v_0\in V$. When $v_0=0$, we suppress the notation and write: $B(\rho)\eqdef B(\rho,0)$.
        \item  
        For any subset \( K \subseteq V \) and number \( \varepsilon > 0 \), an \emph{\( \varepsilon \)-covering} of \( K \) is a finite set \( \{v_1, v_2, \dots, v_N\} \subseteq K \) such that:
        \[
        K \subseteq \bigcup_{i=1}^N B(\varepsilon, v_i)
        \]
        \item 
        The \emph{$\varepsilon$-covering number} \( \mathcal{N}(\varepsilon, K) \) of a compact set \( K \subseteq V \) is defined as:
        \[
        \mathcal{N}(\varepsilon,K)
        \eqdef
        \min \{ N : \exists \text{ an $\varepsilon$-covering of $K$ with size $N$} \}
        < \infty.
        \]
        If \( K=B(\rho) \) is a ball in \( V \) with radius \( \rho \), we suppress the notation of $K$ and write:
        \[
        \mathcal{N}(\varepsilon,\rho) = \mathcal{N}(\varepsilon, B(\rho)).
        \]
    \end{enumerate}
\end{definition}

\subsection{The Proof}
\label{s:Proof_Main_UB__ss:Proof}
We begin by first quantifying the generalization bounds of MLPs induced from their parameter spaces.  Though this may yield loose bounds in general, due to inherent symmetries of MLPs\footnote{E.g. $(-1,1)^{\top}\operatorname{ReLU}((1,-1)x)=(1,-1)^{\top}\operatorname{ReLU}((-1,1)x)$.} this seems to be the most natural approach for LoRAs, which operate directly on the MLP parameter space.
\subsubsection{Step 1 - Probabilistic Bounds on Lipschitz Regularity of Randomized LoRA Class}
\label{s:Proof_Main_UB__ss:Proof___sss:CoveringNumberBound}

We begin by probabilistically controlling the \textit{random} Lipschitz constant of our \textit{random} LoRAs.

\begin{lemma}[High Probability Bound on Lipschitz Regularity of Parameter-to-LoRA Map]
\label{lem:Lipschitz_bound_LoRA}
In Setting~\ref{setting}.
%
%
%
Then with probability at least $1-\epsilon$ (over the random initialization on $(B^{(t)})_t$), the Lipschitz constant $L_{LoRA}^{\theta_{pre}}$ of the (random) LoRA map in~\eqref{eq:LoRA} is at-most 
\[
        L_{LoRA}^{\theta_{pre}}
    \le 
        2^{c_2T}\big(
        \underbrace{M\nu\sqrt{2r\log(2W/\epsilon)}}_{R}
        +
        \underbrace{\|\theta_{pre}\|_{\infty}}_{R_0}
    \big)^{c_2T}.
\]
where $\nu>0$ is the random initialization scale of $B^{(t)}_{ij}\sim\mathcal{N}(0,\nu^2)$.
\end{lemma}
\begin{proof}
\noindent\textbf{Step 1 - Local Lipschitz Regularity of Parameterization Map}\\
Fix $R>0$ (to be settled retroactively).  
If $\sigma$ is either: a bounded Lipschitz function or the $\operatorname{ReLU}$ function then~\cite[Theorem 3]{PetrovaWojtaszczyk_LimitationsMLPs_JMLR_2023}, there is a constant $0\le c_1\eqdef c_1(d,T,W)\le c_2\eqdef c_2(d,T,W)>0$, depending only on $d$, $T$, and on $W$, such that the restriction of the map $P_{T,W}^{\theta_{pre}}(\theta)$ to the $\infty$-norm ball $B_{\infty}(R)\eqdef \{\theta\in\R^p:\|\theta\|_\infty \leq R\}$ is $
L_{LoRA}^{\theta_{pre}}
$-Lipschitz with
\begin{equation}
\label{eq:bound_LipschitzConstantReLUParmaeterization}
    c_1\, T(1+\log_2(R+ \|\theta_{pre}\|_{\infty}))
        \le 
    \log_2(L_{LoRA}^{\theta_{pre}})
        \le 
    c_2\, T(1+\log_2(R + \|\theta_{pre}\|_{\infty} ))
\end{equation}
where we have used the fact that the network weights are at-most $R$ (matrix parameterized by LoRA) plus the initial weights $\theta_{pre}$.

\noindent\textbf{Step 2 - Diameter of LoRA Parameter Space}
\hfill\\
For all $t=1,\dots,T+1$, $i=1,...,W$, note that the column $B^{(t)}_{i\cdot}\sim\mathcal{N}(0,\nu^2\mathbf{I}_r)$ is a Gaussian random vector, hence we have the concentration: $\forall \delta>0$,
\[
    \mathbb{P}\left[\left|\sum_{k=1}^r B^{(t)}_{ik}\right| > \delta\right]
    \leq
    2\exp\left(-\frac{\delta^2}{2\nu^2r}\right)
\]
In other words, with probability at least $1-\epsilon$, it holds that 
\[
    \left|\sum_{k=1}^r B^{(t)}_{ik}\right| \leq \nu\sqrt{2r\log(2/\epsilon)}.
\]
Denote $\|{W}\|_\infty\eqdef \max_{ij} |W_{ij}|$ for any matrix $W$.
Fix $M>0$ and assume that $\|A^{(t)}\|_\infty \leq M$ throughout the training for all $t=1,...,T+1$. Then for all $j=1,...,W$, we have
\[
\left|\left[B^{(t)}A^{(t)}\right]_{ij} \right|
=
\left|\sum_{k=1}^r B^{(t)}_{ik} A^{(t)}_{kj}\right|
\leq
M \left|\sum_{k=1}^r B^{(t)}_{ik}\right|.
\]
By union bound on $i=1,...,W$, with probability at least $1-\epsilon$, it holds that
\[
    \left|\left[B^{(t)}A^{(t)}\right]_{ij} \right|
    \leq M\nu\sqrt{2r\log(2W/\epsilon)},\quad
    \forall i,j,t.
\]
Hence we have, with probability at least $1-\epsilon$,
\begin{equation} \label{eq:LoRA:bound}
    \left\| (B^{(t)}A^{(t)})_{t=1}^T \right\|_\infty
    = \max_t \left\| B^{(t)}A^{(t)} \right\|_\infty
    = \max_{t,i,j} \left|\left[B^{(t)}A^{(t)}\right]_{ij} \right|
    \leq M\nu\sqrt{2r\log(2W/\epsilon)} =: R.    
\end{equation}

Hence we take $R=M\nu\sqrt{2r\log(2W/\epsilon)}$ and we assume the event in Eq. (\ref{eq:LoRA:bound}) holds in the remaining of the proof. \footnote{If we set $L^2$ constraints on $A^{(t)}$ instead of $L^{\infty}$ norm, the constraint size $M$ can be put inside the $\log$ to further improve the bound.}

\noindent\textbf{Step 3 - Lipschitz Constant of LoRA Parameterization}
\hfill\\
Combining~\eqref{eq:bound_LipschitzConstantReLUParmaeterization} and~\eqref{eq:LoRA:bound} we find that
\begin{align*}
            L_{LoRA}^{\theta_{pre}}
     \le 
        2^{c_2\, T(1+\log_2(R + \|\theta_{pre}\|_{\infty} ))}
 \le 
    2^{c_2T}\big(
        M\nu\sqrt{2r\log(2W/\epsilon)}
        +
        \|\theta_{pre}\|_{\infty}
    \big)^{c_2T}.
\end{align*}
\end{proof}

\subsubsection{Step 2 - Probabilistic Bounds on Covering Number of Randomized LoRA Class}
Having obtained a probabilistic estimate of the Lipschitz constant for our random LoRA parameterization map, we derive a probabilistic bound on the covering number of the associated randomized function class; conditioned on the event depending on the draw of the randomized LoRA parameters:
\begin{equation}
\label{eq:conditional_control__LipschitzConstant}
    \mathcal{B}_{\epsilon}
\eqdef 
    \biggl\{
        \omega\in \Omega
    :
            L_{LoRA}^{\theta_{pre}}
        \le 
            2^{c_2T}
            \big(
                M\nu\sqrt{2r\log(2W/\epsilon)}
            +
                \|\theta_{pre}\|_{\infty}
            \big)^{c_2T}
    \biggr\}
\end{equation}
Throughout this paper, the normed space under consideration is typically the parameter space $\R^p$, where the finite collection of model parameters is flattened and concatenated into a single vector. Unless otherwise stated, the norm used is the $\infty$-norm.
We begin with a minor extension—requiring only a brief comment—of a classical covering estimate for balls in normed spaces; see, for instance, \cite[Proposition 15.1.3]{lorentz1996constructive}. The integer $p$ denotes the total number of parameters defining the LoRA mapping.
\begin{lemma}[Covering Number Estimates - Parameter Space]
For each $\varepsilon,\rho>0$, the $\varepsilon$-covering number $\mathcal{N}(\varepsilon,\rho)$ of the closed ball $B(\rho)\subseteq (\R^p,\|\cdot\|_\infty)$ satisfies
\[
    \mathcal{N}(\varepsilon,\rho) \leq \big(\rho \, 2^{-\lfloor \log_2(\varepsilon)\rfloor}\big)^{p}.
\]
\end{lemma}
\begin{proof}
The number of cubes of side-length $2^{-\lfloor \log_2(\varepsilon/L)\rfloor}$ in $[0,\rho]^d$ is precisely $\big(\rho 2^{-\lfloor \log_2(\varepsilon/L)\rfloor}\big)^p$.
\end{proof}

\begin{lemma}[Covering Number Estimates - Function Space]
\label{lem:Mapping_Bound}
Fix $r>0$ and $p\in \mathbb{N}_+$.
Let $(\mathcal{F},\mu)$ be a metric space of functions and $P:(\mathbb{R}^p,\|\cdot\|_\infty)\to
(\mathcal{F},\mu)$ be a ``parameterization map'' with Lipschitz constant at most $L$. Let $K=P(B(\rho))\subseteq (\mathcal{F},\mu)$ be the image of the ball $B(\rho)\subset (\mathbb{R}^p,\|\cdot\|_\infty)$. Then, for every $\varepsilon>0$ the $\varepsilon$-covering number $\mathcal{N}(\varepsilon,K)$ of $K\subseteq (\mathcal{F},\mu)$ is bounded above by
\[
\mathcal{N}(\varepsilon,K)
\le 
\big(\rho\, 2^{-\lfloor \log_2(\varepsilon/L)\rfloor}\big)^{p}
\]
\end{lemma}
\begin{proof}
Let $\{v_i\}_{i=1}^N$ be an $\varepsilon/L$-covering of $B(\rho)\subset(\R^p,\|\cdot\|_\infty)$; where $N\eqdef \mathcal{N}(\varepsilon/L, \rho)$.  
For each $i=1,\dots,N$ define $f_i = P(v_i)$.
Hence for $i,j\in\{1,...,N\}$, 
\begin{equation}
\label{eq:covering_bound}
        \mu(f_i,f_j)
    \le 
        \mu(P(v_i),P(v_j))
    \le 
        L\|x_i - x_j\|_\infty
    \le 
            L\varepsilon/L
        = \varepsilon.
\end{equation}
Hence $\{f_1,...,f_N\}$ is an $\varepsilon$-covering of $K$. By Eq. ~\eqref{eq:covering_bound}, we have
\[
\mathcal{N}(\varepsilon,K) \leq N = \mathcal{N}(\varepsilon/L,\rho) \leq \big(\rho\, 2^{-\lfloor \log_2(\varepsilon/L)\rfloor}\big)^{p}.
\qedhere
\]
\end{proof}




In what follows, we use $1_{\mathbb{R}^D}$ to denote the identity map on $\mathbb{R}^D$.
\begin{lemma}[Lipschitz Stability of Currying]
\label{lem:Lip_Loss}
Let $L_{\ell}\ge 0$ and $\ell:\mathbb{R}^D\times \mathbb{R}^D$.  Then the map 
\begin{align*}
\Lambda:
\big(C(\mathbb{R}^d,\mathbb{R}^D),\|\cdot\|_{\infty}\big)
& \rightarrow
\big(C(\mathbb{R}^d\times \mathbb{R}^D,[0,1]),\|\cdot\|_{\infty}\big)
\\
f &  \mapsto \ell\circ (f\times 1_{\mathbb{R}^D})
\end{align*}
is $L_{\ell}$-Lipschitz.
\end{lemma}
\begin{proof}
Let $x\in \mathbb{R}^d$, $y\in \mathbb{R}^D$, and $f,\tilde{f}\in C(\mathbb{R}^d,\mathbb{R}^D)$.  Then,
\allowdisplaybreaks
\begin{align*}
    \big|
            \ell\big(f(x),y\big)
        -
            \ell\big(\tilde{f}(x),y\big)
    \big|
\le &
    L_{\ell}
    \,
    \big\|
            \big(f(x),y\big)
        -
            \big(\tilde{f}(x),y\big)
    \big\|_2
\\
= &
    L_{\ell}
    \,
        \big\|(f(x)-\tilde{f}(x),y-y)\big\|_2
\\
= &
    L_{\ell}
    \,
        \sqrt{\big\|f(x)-\tilde{f}(x)\big\|^2+\|y-y\big\|_2^2}
\\
= &
    L_{\ell}
    \,
    \big\|f(x)-\tilde{f}(x)\big\|
\\
\le &
    L_{\ell}
    \,
    \sup_{x\in \mathbb{R}^d}\,
    \big\|f(x)-\tilde{f}(x)\big\|
\\
= &
    L_{\ell}
    \,
    \|f-\tilde{f}\|_{\infty}
.
\end{align*}    
Thus, for any $f,\tilde{f}\in C(\mathbb{R}^d,\mathbb{R}^D)$ we have
\[
        \big|
                \Lambda(f)
            -
                \Lambda(\tilde{f})
        \big|_{\infty}
    =
    \sup_{(x,y)\in \mathbb{R}^d\times \mathbb{R}^D}
        \big|
                \ell\big(f(x),y\big)
            -
                \ell\big(\tilde{f}(x),y\big)
        \big|
    \le
        L_{\ell}
        \,
        \|f-\tilde{f}\|_{\infty}
.
\]
Consequentially, $\Lambda$ is $L_{\ell}$-Lipschitz.
\end{proof}

\begin{lemma}[Covering Number Bounds: Random LoRA Class]
\label{lem:Covering_FunctionSpace}
Consider the setting of Lemma~\ref{lem:Lipschitz_bound_LoRA}. 
For each random initialization $\omega\in\Omega$ on the non-trainable parameters $(B^{(t)})_{t=1}^{T+1}$, define the LoRA function space $\mathcal{F}$ 
equipped with the metric induced by the uniform norm $\|\cdot\|_{\infty}$:
\[
    \mathcal{F} \eqdef \mathcal{F}(\omega,M)
    =
    \{ \text{LoRA}(\omega,(A^{(t)})_{t=1}^{T+1})\in C([0,1]^d,\R^D) : |A^{(t)}_{ij}| \leq M,\ \forall t,i,j \}.
\]
Then, for every $\epsilon,
\varepsilon>0$ the $\varepsilon$-covering number $\mathcal{N}(\varepsilon,\mathcal{F})$ of $\mathcal{F}$
satisfies
\begin{equation}
\label{eq:Prob_of_Covering}
        \mathbb{P}\Big(
            \mathcal{N}(\varepsilon,\mathcal{F})
            \leq
            \left((2R+2R_0)^{cT+1} / \varepsilon \right)^q
        \Big)
    \ge 
        \mathbb{P}(\mathcal{B}_{\epsilon})
    \ge 
        1-\epsilon
\end{equation}
where $c=c_2$, $R=M\nu\sqrt{2r\log(2W/\epsilon)}$ and $R_0=\|\theta_\text{pre}\|_\infty$ as in Lemma~\ref{lem:Bound_t:CondLipschitz}.
\hfill\\
\noindent
Moreover, when the loss function $\ell$ is $1$-Lipschitz, composing $\mathcal{F}$ with $\Lambda$ in Lemma \ref{lem:Lip_Loss} yields
\begin{equation}
\label{eq:Prob_of_Covering__fullrefined}
        \mathbb{P}\Big(
                \mathcal{N}(\varepsilon,\Lambda(\mathcal{F}))
            \leq
                \left((2R+2R_0)^{cT+1} / \varepsilon \right)^q
        \Big)
    \ge 
        \mathbb{P}(\mathcal{B}_{\epsilon})
    \ge 
        1-\epsilon
\end{equation}
\end{lemma}
\begin{proof}
This is a direct consequence of Lemmata Lemmata~\ref{lem:Lipschitz_bound_LoRA} and~\ref{lem:Mapping_Bound}: with probability at least $1-\epsilon$, we have
\begin{align*}
    \mathcal{N}(\varepsilon,\mathcal{F})
    &\leq
    \left(R\cdot 2^{-\lfloor \log_2(\varepsilon/L_\text{LoRA}^{\theta_\text{pre}}) \rfloor}\right)^q\\
    &\leq
    \left(2R\cdot L_\text{LoRA}^{\theta_\text{pre}} / \varepsilon \right)^q\\
    &\leq 
    \left(2R\cdot (2R+2R_0)^{cT} / \varepsilon \right)^q\\
    &\leq 
    \left((2R+2R_0)^{cT+1} / \varepsilon \right)^q
.
\qedhere
\end{align*}
\end{proof}

\subsubsection{Step 3 - Generalization Bounds Conditioned on Covering Number Bounds}

Next, we control the randomized generalization gap $\mathbf{G}$, as defined in~\eqref{eq:rand_gen}, conditioned on the event $\mathcal{B}_{\epsilon}$, as defined in~\eqref{eq:conditional_control__LipschitzConstant}.  Upon conditioning on the right realizations, which happen with high probability, our covering number bounds give us access to the Dudley's integral estimate; see e.g.~\citep[Corollary 2.2.8]{varderVaartWellner_EmpiricalProcessesBook_1996}.

\begin{lemma}[Conditional Generalization Bounds for ``Derandomized'' LoRAs]
\label{lem:Bound_t:CondLipschitz}
Consider the setting of Lemma~\ref{lem:Lipschitz_bound_LoRA}.  Let  $G
    \eqdef
    \sup_{f\in \mathcal{F}}\,
        \big|
            \mathcal{R}(f)
            -
            \mathcal{R}^N(f)
        \big|$
be the generalization error. The following holds
\begin{equation}
\label{eq:generalization_conditional__GoodCovering}
    \mathbb{P}\biggl(
                \mathbf{G} 
            \leq 
                4\min \left\{  1, 
            \sqrt{\rho}
            \frac{6\sqrt{A}}{\sqrt{N}}  \right\} + \sqrt{\frac{8\log (2/\epsilon)}{N}}
    \Big|
        \,
        \mathcal{B}_{\epsilon}
    \biggr)
\ge 
    1-\epsilon
\end{equation}
where $A=(cT+1)\log(2R+2R_0)$, $R=M\nu\sqrt{2r\log(2W/\epsilon)}$ and $R_0=\|\theta_\text{pre}\|_\infty$ and $\rho\eqdef \rho(q,N)\eqdef q$.
\end{lemma}
\begin{proof}

Let $\hat{\mathrm{R}}(\Lambda\circ \mathcal{F})$ denote the \textit{empirical Rademacher complexity}, computed from the $N$ samples $\{(X_n,Y_n)\}_{n=1}^N$ for the class $\Lambda\circ \mathcal{F}\eqdef \{\ell\circ (f\times 1_{\mathbb{R}^D}:\, f\in \mathcal{F}\}$.
By~\cite[Theorem 8]{bartlett2002rademacher}, for every $\epsilon\in (0,1]$, the following holds with probability at-least $1-\epsilon$
\begin{equation}
\label{eq:Rademacher_Bound}
    \mathbf{G}
\le 
    \underbrace{
        2\,\hat{\mathrm{R}}(\Lambda\circ \mathcal{F})
    }_{\term{t:ERC}}
    +
    \sqrt{\frac{
        8\log(2/\epsilon)
    }{
        N
    }}.
\end{equation}

To bound the term (\ref{t:ERC}) in Eq. (\ref{eq:Rademacher_Bound}), we apply~\cite[Proposition 5.17]{Wainwright_2019_HDStat} to obtain:
\[
            \hat{\mathrm{R}}(\Lambda\circ \mathcal{F})
     \le 
    \inf_{t\in [0,c^{\star}/2]}\, 
        4t
    + 
        \frac{12}{\sqrt{N}}
            \int_t^{c^{\star}/2} 
                \sqrt{\log\mathcal{N}(\varepsilon,\Lambda\circ\mathcal{F})}
                \,
            d\varepsilon
\]
where $c^{\star}:=\sup_{f\in \mathcal{F}}\, \frac1{m}\,(\sum_{i=1}^m\,(\Lambda\circ f)^2(z_i))^{1/2}$. Note that the loss $\ell$ takes value between $[0,1]$, hence $c^*\leq 1$ and we have:
\begin{equation} \label{eq:rademacher}
    \hat{\mathrm{R}}(\Lambda\circ \mathcal{F})
     \le        
    \inf_{t\in [0,1/2]}\, 
        4t
        \,
    + 
        \frac{12}{\sqrt{N}}
            \int_t^{1/2} 
                \sqrt{\log\mathcal{N}(\varepsilon,\Lambda\circ\mathcal{F})}
                \,
            d\varepsilon.
\end{equation}
Conditioned on the event $\mathcal{B}_{\epsilon}$, we have
\[
    \sqrt{\log\mathcal{N}(\varepsilon,\Lambda\circ\mathcal{F})}
    \leq 
    \sqrt{\log\left((2R+2R_0)^{cT+1} / \varepsilon \right)^q}
    \leq 
    \sqrt{q(A+\log\varepsilon^{-1})}.
\]
where we denote $A = {(cT+1)\log(2R+2R_0)}$.
To bound the RHS of Eq. (\ref{eq:rademacher}), We will find the optimum $t^*\in[0,1/2]$ of the function $g(t)= 4t + \frac{12}{\sqrt{N}} \int_t^{1/2} \sqrt{q(A+\log\varepsilon^{-1})} \,d\varepsilon.$ By first order optimality condition:
\begin{align*}
    g'(t) &= 0\\
    4 - \frac{12}{\sqrt{N}}\sqrt{q(A+\log t^{-1})} &= 0\\
    t &= \exp \left( A - \frac{N}{9q} \right)
\end{align*}
If $N^2\gg q$, then $t^*\approx 0$. In such case, we have
\begin{align*}
    \hat{\mathrm{R}}(\Lambda\circ \mathcal{F})
    &\leq
    \frac{12}{\sqrt{N}} \int_0^{1/2} \sqrt{q(A+\log\varepsilon^{-1})}\, d\varepsilon\\
    &\leq
    \frac{12}{\sqrt{N}} \sqrt{\int_0^{1/2} q(A+\log\varepsilon^{-1})\, d\varepsilon}\\
    &=
    \frac{12}{\sqrt{N}} \sqrt{\frac{q}{2}(A+\log 2 +1)}\\
    &\leq
    \frac{12}{\sqrt{N}} \sqrt{qA}.
\end{align*}
If $N^2\ll q$, then $t^*\approx 1/2$. In such case, we have $\hat{\mathrm{R}}(\Lambda\circ \mathcal{F}) \leq 2$.
In any case, we have 
\[
\hat{\mathrm{R}}(\Lambda\circ \mathcal{F}) \leq \min \left\{  2, \frac{12\sqrt{qA}}{\sqrt{N}}  \right\}.
\]
Plug in the above result to Eq. (\ref{eq:Rademacher_Bound}) the following holds with conditional probability, conditioned on $\mathcal{B}_{\epsilon}$, at least $1-2\epsilon$, over the random draw of $\{(X_n,Y_n)\}_{n=1}^N$ the randomized generalization error is bounded by:
\[
        \mathbf{G} 
    \leq
        4\min \left\{  1, \frac{6\sqrt{qA}}{\sqrt{N}}  \right\} + \sqrt{\frac{8\log (2/\epsilon)}{N}}
    =
        4\min \left\{  1, 
    \sqrt{\rho}
    \frac{6\sqrt{A}}{\sqrt{N}}  \right\} + \sqrt{\frac{8\log (2/\epsilon)}{N}}
\]
where $A=(cT+1)\log(2R+2R_0)$, $R=M\nu\sqrt{2r\log(2W/\epsilon)}$, $R_0=\|\theta_\text{pre}\|_\infty$, and we recall that $\rho\eqdef q$.
\end{proof}

\subsubsection{Step 4 - Generalization Bounds Conditioned on Covering Number Bounds}
\label{s:Proof_Main_UB__ss:Proof___sss:PuttingTogether}

We are now ready to deduce our main theorem.

\begin{proof}[{Proof of Theorem~\ref{thrm:Main_UB}}]
Let $\epsilon\in (0,1]$ be given; to be determined retroactively.  
As in~\eqref{eq:Prob_of_Covering__fullrefined} and~\eqref{eq:generalization_conditional__GoodCovering} we the following quantities
\[
        G^{\star}
    \eqdef 
        4\min \left\{  1, 
            \sqrt{\rho}
            \frac{6\sqrt{A}}{\sqrt{N}}  \right\} + \sqrt{\frac{8\log (2/\epsilon)}{N}}
\mbox{ and }
        \mathcal{N}^{\star}
    \eqdef 
        \mathbb{P}\Big(
                \mathcal{N}(\varepsilon,\Lambda(\mathcal{F}))
            \leq
                \left((2R+2R_0)^{cT+1} / \varepsilon \right)^q
.
\]  Then, the following holds
\allowdisplaybreaks
\begin{align*}
     \mathbb{P}\big(
            \mathbf{G}
        \le 
            G^{\star}
     \big)
= &  
    \mathbb{P}\big(
            \mathbf{G}
        \le 
            G^{\star}
        \big|
            \mathcal{B}_{\epsilon}
     \big)
     \mathbb{P}(\mathcal{B}_{\epsilon})
+  
    \mathbb{P}\big(
            \mathbf{G}
        \le 
            G^{\star}
        \big|
            \mathcal{B}_{\epsilon}
     \big)
     \mathbb{P}(\Omega\setminus \mathcal{B}_{\epsilon}^c)
\\
& \ge 
    \mathbb{P}\big(
            \mathbf{G}
        \le 
            G^{\star}
        \big|
            \mathcal{B}_{\epsilon}
     \big)
     \mathbb{P}(\mathcal{B}_{\epsilon})
\\
\numberthis
\label{eq:conditionedal_bound}
& \ge 
    (1-\epsilon)
     \mathbb{P}(\mathcal{B}_{\epsilon})
\\
\numberthis
\label{eq:conditioneding_prob}
& \ge 
    (1-\epsilon)
     (1-\epsilon)
\end{align*}
where~\eqref{eq:conditionedal_bound} holds by Lemma~\ref{lem:Bound_t:CondLipschitz} and~\eqref{eq:conditioneding_prob} holds by Lemma~\ref{lem:Covering_FunctionSpace}.  This completes our proof.  Let $0<\delta \le 1$ and set $1-\delta = (1-\epsilon)^2$.  Solving for $\epsilon$ in $(0,1]$ yields the conclusion; namely, with $\epsilon = 1-\sqrt{\delta}$.
\end{proof}

\section{Proof of Optimality: Matching Lower Bounds on the Sample Complexity}
\label{s:Proofs__ss:LB}

\begin{proof}[{Proof of Theorem~\ref{thrm:LBs}}]
Fix $M>0$ and $\delta \in (0,1)$.  
Let $p\in (0,1)$, $\mathbb{P}_X$ any sampling distribution satisfying Assumption~\ref{ass:BreakingDistribution} (of which we know on exists by Example~\ref{ex:breaking}).  
Recall that $\mathbb{P}=\mathbb{P}_X\otimes\delta_0$.  
Suppose also that $\sigma=\operatorname{ReLU}$, (for simplicity) $b^{(1)}=0$ for each $t=1,\dots,T+1$ (where we write $0$ for the appropriate-dimensional zero vector), and let $\ell(\hat{y},y)\eqdef |\hat{y}-y|$.
\hfill\\
\noindent\textbf{Step 1 - The MLP Construction:}
\hfill\\
For each $t=1,\dots,T+1$ suppose that $d_{t+1} \ge d_t$ let $\bar{I}_{d_{t+1},d_t}\eqdef I_{d_t}\oplus \mathbf{0}_{d_{t+1}-d_t}$ denote the direct-sum of the $d_t\times d_t$ identity matrix $I_{d+t}$ and the $(d_{t+1}-d_t)\times (d_{t+1}-d_t)$ zero matrix $\mathbf{0}_{d_{t+1}-d_t}$ if $d_{t+1}>d_t$ and $\bar{I}_{d_{t+1},d_t}\eqdef I_{d_t}$ if $d_{t+1}=d_t$.  
For each $t=1,\dots,T+1$ define $A^{(t)}$ 
\begin{equation}
\label{eq:lower_bound_weighting}
        A^{(t)}
    \eqdef 
        {B^{(t)}}^{\dagger}\,(\bar{I}_{d_{t+1},d_t} - W^{(t)})
\end{equation}
where ${B^{(t)}}^{\dagger}$ denotes the Moore-Penrose pseudo-inverse of $B^{(t)}$ (which we recall, is a right-inverse of $B^{(t)}$ if $B^{(t)}$ has full row-rank).

By construction, we have that: for each $(x,y)\in \operatorname{supp}(\mathbb{P})$ 
$\operatorname{ReLU}\bullet ((W^{(1)}+B^{(1)}A^{(1)})x^{(1)} + b^{(0)}) \ge 0$ (where $\ge$ is applied componentwise, thus we are considering the product order on $\mathbb{R}^{d_2}$); whence, the following holds
\begin{align}
        \operatorname{ReLU}\bullet ((W^{(1)}+B^{(1)}A^{(1)})x^{(1)} + b^{(0)}) 
    & =
        ((W^{(1)}+B^{(1)}A^{(1)})x^{(1)} + b^{(0)}) 
\\
    & =
        (W^{(1)}+B^{(1)}A^{(1)})x^{(1)} 
\\
    & = 
        \big(W^{(1)} + 
        B^{(1)}
        {B^{(1)}}^{\dagger}\,(\bar{I}_{d_{t+1},d_t} - W^{(1)})\big) x^{(1)}
    =
        x^{(1)}
.
\end{align}
where the right-hand inequality held by definition of $A^{(t)}$ in~\eqref{eq:lower_bound_weighting} (assuming that $B^{(t)}$ has \textit{full rank}).

Iterating the above argument and appealing to the recursive representation of $\hat{f}$ in~\eqref{eq:LoRA}, we deduce that
\[
    \hat{f}(x) = x
\]
holds for each $(x,y)\in \operatorname{supp}(\mathbb{P})$.  By our specifications on $\ell$ and $\mathbb{P}$ we have that: for each $(x,y)\in \operatorname{supp}(\mathbb{P})$
\[
\ell(\hat{f}(x),y) = |\hat{f}(x)-0| = |x-0| = x
.
\]
Since $(X_1,Y_1),\dots,(X_N,Y_N)\sim \mathbb{P}$ are i.i.d.\ random vectors satisfying with law $\mathbb{P}$ then, we have that
\begin{equation}
\label{eq:LB_Risk_and_EmpiricalRisk}
        \mathcal{R}(\hat{f}) = \mathbb{E}_{X\sim \mathbb{P}_X}[X]
    \mbox{ and }
        \mathcal{R}^N(\hat{f}) = \frac1{N}\, \sum_{n=1}^N\,X_n
.
\end{equation}
Consequentially, we find that
\begin{equation}
\label{eq:LB_via_CLTSetup}
\sup_{f\in \mathcal{F}}\,
        \big|
            \mathcal{R}(f)
            -
            \mathcal{R}^N(f)
        \big|
\ge 
    \big|
        \mathcal{R}(\hat{f})-\mathcal{R}^N(\hat{f})
    \big|
=
    \Big|
            \mathbb{E}_{X\sim \mathbb{P}_X}[X]
        -
            \frac1{N}\, \sum_{n=1}^N\,X_n
    \Big|
.
\end{equation}

\hfill\\
\noindent\textbf{Step 2 - Weight Admissibility:}
\hfill\\
\textit{We now need to verify that the choice of weights in~\eqref{eq:lower_bound_weighting} is indeed valid, meaning that the operator norm of the trainable LoRA factors $A^{(1)},\dots,A^{(T+1)}$ are indeed bounded above by $M$ (with some acceptable probability $1-\delta/2$).}
\hfill\\
Assuming that $M$ and $R$ are large enough so that $A^{(t)}$ is an admissible choice with positive probability.
We directly see that: for each $t=1,\dots,T+1$ 
\begin{align}
\label{eq:operator_norm_bound}
        \|A^{(t)}\|_{op}
    & \le
        \|
            {B^{(t)}}^{\dagger}
        \|_{op}
        \,
        \|
            (\bar{I}_{d_{t+1},d_t} - W^{(t)})
        \|_{op}
\\
\nonumber
    & \le
        \|
            {B^{(t)}}^{\dagger}
        \|_{op}
        \,
        \big(
            1
            +
            \|W^{(t)}\|_{op}
        \big)
\\
\label{eq:singular_value_min_bound__ok}
    & =
        \frac1{
            s_{\min}(B^{(t)})
        }
        \,
        \big(
            1
            +
            \|W^{(t)}\|_{op}
        \big)
\end{align}
where, for each $t=1,\dots,T+1$, $s_{\min}(B^{(t)})$ denotes the smallest singular value of $B^{(t)}$ respectively.
Since, for each $t=1,\dots,T+1$, each $B^{(t)}$ has i.i.d.\ standard normal entries.
Consider the constant $C_{\operatorname{pre}}\ge 0$, depending only on the pre-trained weight matrices $\{W^{(t)}\}_{t=0}^T$ by
\[
        \tilde{C}_{\operatorname{pre}}
    \eqdef 
        \max_{t=0,\dots,T}\, \|W^{(t)}\|_{op}
.
\]
Now, the upper-bound in~\eqref{eq:operator_norm_bound}-\eqref{eq:singular_value_min_bound__ok} becomes
\begin{align}
\label{eq:singular_value_min_bound}
        \|A^{(t)}\|_{op}
    \le
        \frac1{
            s_{\min}(B^{(t)})
        }
        \,
        \big(
            1
            +
            \|W^{(t)}\|_{op}
        \big)
    \le
        \frac1{
            s_{\min}(B^{(t)})
        }
        \,
        (1+\tilde{C}_{\operatorname{pre}})
    =
        \frac{
            C_{\operatorname{pre}}
        }{
            s_{\min}(B^{(t)})
        }
\end{align}
where $C_{\operatorname{pre}}\eqdef \tilde{C}_{\operatorname{pre}}+1 > 0$.
Then, by Gordon's Theorem, as formulated in~\citep[Exercise 7.3.4]{vershynin2010introduction}, we have that: for each $t=1,\dots,T+1$ we have that and every $\eta>0$
\begin{equation}
\label{eq:Gordon_bound}
    \mathbb{P}\big(
            \sqrt{d_{t+1}} 
            -
            \sqrt{r} 
            -
            \eta
        >
            s_{\min}(B^{(t)})
    \big)
    \le 
    2\,e^{-c\eta^2}
.
\end{equation}
Taking a union bound, we have that: for every $\eta>0$
\begin{equation}
\label{eq:Gordon_bound__union_bound}
        \mathbb{P}\Big(
            \neg\,
            (\forall \, t=1,\dots,T+1)
            \,
                \sqrt{d_{t+1}} 
                -
                \sqrt{r} 
                -
                \eta
            \le 
                s_{\min}(B^{(t)})
        \Big)
    \le
        2 (T+1)e^{-c\eta^2}
.
\end{equation}
By the contracting widths condition in~\eqref{eq:dimension_condition}, the interval $(0,\eta_{\star})$ is non-empty.  Whence, for every $0<\eta<\eta_{\star}$ the left-hand side of the following bound quantity, in the sense that it is strictly positive
\begin{equation}
\label{eq:M_prep}
        \min\limits_{t=0,\dots,T+1}
        \,
            \sqrt{d_{t+1}}
            -
            \sqrt{r} 
            -
            \eta
    \le 
        \min\limits_{t=0,\dots,T+1}\,
            s_{\operatorname{min}}(B^{(t)})
.
\end{equation}
Thus, for each $\eta\in (0,\eta_{\star})$ the following quantity is a well-defined positive number
\begin{equation}
\label{eq:M_defined}
    0
<
    M(\eta)
\eqdef 
    \frac{1}{
        \min\limits_{t=0,\dots,T+1}
        \,
            \sqrt{d_{t+1}}
            -
            \sqrt{r}
            -
            \eta
    }
    =
    \max\limits_{t=0,\dots,T+1}
        \,
        \frac{1}{
            \sqrt{d_{t+1}}
            -
            \sqrt{r}
            -
            \eta
        }
<
    \infty
.
\end{equation}
Together,~\eqref{eq:M_defined},~\eqref{eq:Gordon_bound__union_bound}, and~\eqref{eq:Gordon_bound} imply that: for each $0<\eta<\eta_{\star}$ we have
\begin{equation}
\label{eq:non_vaccouns_Mbound__etaform}
        \mathbb{P}\Big(
                0
            <
            \max\limits_{t=0,\dots,T+1}
                \frac1{s_{\min}(B^{(t)})}
            \le
                M(\eta)
            <
                \infty
        \Big)
    \le
        2(T+1)e^{-c\eta^2}
.
\end{equation}
If $\delta$ is small enough, namely if $
2(T+1)e^{-
c
\eta^2}<\delta 
< 2(T+1) 
$ then setting $\frac{\delta}{2}= (T+1)e^{-c\eta^2}$ implies that the condition $0<\eta<\eta_{\star}$ required for the non-vacuousness (positivity) of the left-hand side of~\eqref{eq:M_prep} is fulfilled.  
Recalling our interest in the minimal singular value of each random LoRA factor given by the inequalities in~\eqref{eq:operator_norm_bound}-\eqref{eq:singular_value_min_bound}, then~\eqref{eq:non_vaccouns_Mbound__etaform} can be transcribed as
\begin{equation}
\label{eq:non_vaccouns_Mbound__DeltaForm}
        \mathbb{P}(
            \mathcal{A}
        )
    \ge 
        1- \frac{\delta}{2}
\end{equation}
where we abbreviate the event 
\[
        \mathcal{A}
    \eqdef 
        \Big\{
            \omega\in \Omega
            :
            \,
                    0
                <
                    \max\limits_{t=0,\dots,T+1}
                        \|A^{(t)}(\omega)\|_{op}
                    \le
                        M(\eta)
                        C_{\operatorname{pre}}
                <
                    \infty
        \Big\}
.
\]
\hfill\\
\noindent\textbf{Step 3 - Anti-Concentration:}
\hfill\\
\noindent For each $n=1,\dots,N$, define
\begin{equation}
\label{eq:COV_xi_to_X}
    \xi_n
\eqdef 
    2(X_n - 1)
\end{equation}
note that, $\mathbb{E}[\xi_n]=0$, that $\mathbb{E}[(\xi_n
-\mathbb{E}[\xi_n]
)^2]=1$, and that $\mathbb{E}[\|\xi_n\|^3]\le 1$.
We may thus apply~\citep[Theorem 4.1]{LittlewoodOffordProblem_VershyninRudelson_2008} (with $\kappa=N/4$ and $\alpha=1/3$ as in~\citep[Remark 1]{LittlewoodOffordProblem_VershyninRudelson_2008}) to obtain the following estimate for the \textit{small ball probability}: for every $t>0$
\begin{equation}
\label{eq:SmallBall}
    p_t
\eqdef
    \sup_{v\in \mathbb{R}}
    \,
        \mathbb{P}\biggl(
                \big|
                    \sum_{n=1}^N\,
                        \xi_n
                    -
                    v
                \big|
            \le
                t
        \biggr)
\lesssim 
    \frac{(t+1)}{\sqrt{N}}
\end{equation}
where $\lesssim$ hides absolute constants (i.e.\ constants which do not depend on $N$).
We thus, relatively directly, deduce the following anti-concentration inequality
\begin{equation}
\label{eq:anti_concentration__in_xi}
        \mathbb{P}\biggl(
                \big|
                    \sum_{n=1}^N\,
                        \xi_n
                \big|
            \le
                t
        \biggr)
    =
        \mathbb{P}\biggl(
                \big|
                    \sum_{n=1}^N\,
                        \xi_n
                    -
                    \mathbb{E}[\xi]
                \big|
            \le
                t
        \biggr)
\le
    p_t
\lesssim 
    \frac{(t+1)}{\sqrt{N}}
.
\end{equation}
Using~\eqref{eq:COV_xi_to_X} and~\eqref{eq:anti_concentration__in_xi} to conclude that
\begin{equation}
\label{eq:anti_concentration__in_X}
        \mathbb{P}\biggl(
                \Big|
                    \frac1{N}
                    \,\sum_{n=1}^N\,
                        X_n
                    -
                        \mathbb{E}[X_1]
                \Big|
            \le
                \frac{t}{2N}
        \biggr)
    =
        \mathbb{P}\biggl(
                \Big|
                    \sum_{n=1}^N\,
                        \xi_n
                    -
                    \mathbb{E}[\xi]
                \Big|
            \le
                t
        \biggr)
\le
    p_{t/(2N)}
\lesssim 
    \frac{(t+1)}{\sqrt{N}}
.
\end{equation}
Setting $t=2$, and absorbing a constant factor of $2$ into $\lesssim$, then~\eqref{eq:anti_concentration__in_X} implies that
\begin{equation}
\label{eq:anti_concentration__Nearly_Done}
        \mathbb{P}\biggl(
                \Big|
                    \frac1{N}
                    \,\sum_{n=1}^N\,
                        X_n
                    -
                        \mathbb{E}[X_1]
                \Big|
            >
                \frac1{N}
        \biggr)
\ge
    1- c\,\frac{1}{\sqrt{N}}
\end{equation}
for some absolute constant $c>0$ (i.e.\ not depending on $N$).
We are now prepared to conclude.
\hfill\\
\noindent\textbf{Step 4 - Putting it All Together:}
\hfill\\
Now, from~\eqref{eq:LB_via_CLTSetup},~\eqref{eq:non_vaccouns_Mbound__DeltaForm}, and~\eqref{eq:anti_concentration__Nearly_Done}, and upon conditioning we have that
\allowdisplaybreaks
\begin{align}
    \mathbb{P}\biggl(
        \sup_{f\in \mathcal{F}}\,
            \big|
                \mathcal{R}(f)
                -
                \mathcal{R}^N(f)
            \big|
        >
            \frac{1}{N}
    \biggr)
& \ge 
    \mathbb{P}\biggl(
        \sup_{f\in \mathcal{F}}\,
            \big|
                \mathcal{R}(f)
                -
                \mathcal{R}^N(f)
            \big|
        >
            \frac{1}{N}
    \Big|
        \mathcal{A}
    \biggr)
    \,
    \mathbb{P}(\mathcal{A})
\\
& \ge 
    \mathbb{P}\biggl(
            \Big|
                \frac1{N}
                \,\sum_{n=1}^N\,
                    X_n
                -
                    \mathbb{E}[X_1]
            \Big|
        \ge
            \frac1{N}
    \biggr)
    \,
    \mathbb{P}(\mathcal{A})
\\
& \ge 
    \mathbb{P}\biggl(
            \Big|
                \frac1{N}
                \,\sum_{n=1}^N\,
                    X_n
                -
                    \mathbb{E}[X_1]
            \Big|
        \ge
            \frac1{N}
    \biggr)
    \,
    \Big(
        1- \frac{\delta}{2}
    \Big)
\\
& \ge 
    \Big(
        1- c\,\frac{1}{\sqrt{N}}
    \Big)
    \,
    \Big(
        1- \frac{\delta}{2}
    \Big)
.
\end{align}
This concludes our proof.   
\end{proof}

\section{Standard Rademacher Bounds}
\label{a:Rademchar}
This appendix presents a specific formulation of standard generalization bounds based on the empirical \textit{Rademacher} complexity of a hypothesis class, which, despite being folklore, we were (surprisingly) unable to locate in the literature. 
Given $S = \{x_1, \ldots, x_m\}$, we consider the metric space $(\mathcal{R}_S, L_S^2)$, where the metric $L_S^2$ is defined as
\[
L_S^2(f, g) := \sqrt{\frac{1}{m} \sum_{i=1}^m (f(x_i) - g(x_i))^2}
\]
for all $f, g \in \mathcal{R}_S$.

\begin{theorem}
Consider any function class $\mathcal{F}$ containing functions $f : \R^n \to \R$. 
Let $h$ be the zero function such that $h(x) = 0$ for all $x \in S$. Suppose $B := \sup_{f \in C} L_S^2(f, h)$.
Let $\hat{R}_n(\mathcal{F})$ be the empirical Rademacher complexity of function class $\mathcal{F}$. Then
\begin{equation}
\hat{R}_n(\mathcal{F}) \leq \inf_{\epsilon \geq 0} \left \{ 4 \epsilon + 12 \int_\epsilon^{B} \sqrt{\frac{\log \cN(\tau)}{n}} d\tau \right \}
\end{equation}

Alternatively,
\[
\hat{R}_n(\mathcal{F}) \leq 12 \int_0^B \sqrt{\frac{\log \cN(C, L_S^2, \epsilon)}{n}} \, d\varepsilon.
\]

\end{theorem}
\begin{proof}

Let $k$ be a positive integer. For all $0 \leq j \leq k$, define $\epsilon_j := B \cdot 2^{-j}$ and let $T_j$ be a minimum $\epsilon_j$-cover of $(\cF, L_S^2)$. It follows that $|T_j| = \cN(\cF, L_S^2, \epsilon_j)$. We let $T_0 := \{h\}$ since $L_S^2(f, h) \leq B = \epsilon_0$ for all $f \in \cF$. Note that $\cN(\cF, L_S^2, \epsilon)$ is non-increasing with respect to $\epsilon$, hence $|T_{j-1}| = \cN(\cF, L_S^2,\epsilon_{j-1}) \leq \cN(\cF, L_S^2,\epsilon_j) = |T_j|$ for $0 < j \leq k$. Without loss of generality, we assume $T_k$ is a finite set (and hence all $T_j$'s are finite sets), since otherwise $\cN(\cF, L_S^2,\epsilon)$ is unbounded for $0 \leq \epsilon \leq \epsilon_k$, in which case the integral is also unbounded and the inequality is trivially true.

For each $f \in \cF$ and $0 \leq j \leq k$, let $f_j \in T_j$ be a function such that $f_j$ covers $f$ in $T_j$, that is, $L_S^2(f, f_j) \leq \epsilon_j$. Then we can represent each $f$ by
\[
f = f - f_k + \sum_{j=1}^k (f_j - f_{j-1}),
\]
where $f_0 = h$. Then we have
\[
\hat{R}_S(\cF) = \E_\sigma \left[ \sup_{f \in \cF} \frac{1}{m} \sum_{i=1}^m \sigma_i \left( f(x_i) - f_k(x_i) + \sum_{j=1}^k (f_j(x_i) - f_{j-1}(x_i)) \right) \right].
\]
Applying the triangle inequality, we obtain
\[
\hat{R}_S(\cF) \leq \E_\sigma \left[ \sup_{f \in \cF} \frac{1}{m} \sum_{i=1}^m \sigma_i (f(x_i) - f_k(x_i)) \right] + \sum_{j=1}^k \E_\sigma \left[ \sup_{f \in \cF} \frac{1}{m} \sum_{i=1}^m \sigma_i (f_j(x_i) - f_{j-1}(x_i)) \right].
\]
We consider the first and second terms in the last expression separately. For the first term, recall that the $\sigma_i$’s are random variables taken from $\{-1, +1\}$. Applying the Cauchy-Schwarz inequality, we obtain
\[
\E_\sigma \left[ \sup_{f \in \cF} \frac{1}{m} \sum_{i=1}^m \sigma_i (f(x_i) - f_k(x_i)) \right] \leq \E_\sigma \left[ \sup_{f \in \cF} \frac{1}{m} \sqrt{\sum_{i=1}^m \sigma_i^2 \sum_{i=1}^m (f(x_i) - f_k(x_i))^2} \right].
\]
Since $\sigma_i$ are independent random variables with variance 1, we can simplify this to
\[
\leq \E_\sigma \left[ \sup_{f \in \cF} \sqrt{\frac{1}{m} \sum_{i=1}^m (f(x_i) - f_k(x_i))^2} \right] = \E_\sigma \left[ \sup_{f \in \cF} L_S^2(f, f_k) \right] \leq \epsilon_k.
\]

Now, we consider the second term $\sum_{j=1}^k \E_\sigma \left[ \sup_{f \in \cF} \frac{1}{m} \sum_{i=1}^m \sigma_i (f_j(x_i) - f_{j-1}(x_i)) \right]$. We fix $j$, and define $g_f := f_j - f_{j-1}$. That is, we define a new function $g_f$ for each $f \in \cF$. Let $G \eqdef \{g_f : f \in \cF\}$ be the collection of $g$ functions. It follows that
\[
\E_\sigma \left[ \sup_{f \in \cF} \frac{1}{m} \sum_{i=1}^m \sigma_i (f_j(x_i) - f_{j-1}(x_i)) \right] = \E_\sigma \left[ \sup_{g \in G} \frac{1}{m} \sum_{i=1}^m \sigma_i g(x_i) \right].
\]
Since $f_j \in T_j$ and $f_{j-1} \in T_{j-1}$, we have $|G| \leq |T_j| |T_{j-1}| \leq |T_j|^2$. Since $T_j$ is finite, the set $G$ is also finite. Also, for each $g = g_f \in G$, we have
\[
\sqrt{\sum_{i=1}^m g_f(x_i)^2} = L_S^2(f_j, f_{j-1}) \sqrt{m} \leq (L_S^2(f, f_j) + L_S^2(f, f_{j-1})) \sqrt{m} \leq (3 \epsilon_j) \sqrt{m}.
\]
Thus, $\sup_{g \in G} \sqrt{\sum_{i=1}^m g(x_i)^2} \leq 3 \epsilon_j \sqrt{m}$. Applying Massart’s Lemma to the functions $G$, we obtain
\[
\E_\sigma \left[ \sup_{g \in G} \frac{1}{m} \sum_{i=1}^m \sigma_i g(x_i) \right] \leq 3 \epsilon_j \sqrt{\frac{2 \ln |G|}{m}} \leq 6 \epsilon_j \sqrt{\frac{\ln |T_j|}{m}}.
\]

Combining all inequalities, we get
\[
\hat{R}_S(\cF) \leq \epsilon_k + 6 \sum_{j=1}^k \epsilon_j \sqrt{\frac{\ln |T_j|}{m}}.
\]
Taking the limit as $k \to \infty$, and using the fact that $\cN(\cF, L_S^2,\epsilon) \geq \cN(\cF, L_S^2,\epsilon_j)$ for all $\epsilon_{j+1} \leq \epsilon \leq \epsilon_j$, we arrive at
\[
\hat{R}_S(\cF) \leq 12 \int_0^B \sqrt{\frac{\ln \cN(\cF, L_S^2,\epsilon)}{m}} \, d\epsilon.
\]
\end{proof}



\section{Broader Impact}
As our work only provides a \textit{positive} theoretical foundation for commonly adopted machine learning practice, it can only have a positive social impact, as it does not introduce any new AI technology.  On the contrary, we only provide peace of mind that what is deployed in the wild works.

\end{document}